\documentclass[final,12pt]{msml2021} 

\usepackage{hyperref}       
\usepackage{url}            
\usepackage{booktabs}       
\usepackage{amsfonts}       
\usepackage{nicefrac}       
\usepackage{microtype}      
\usepackage{enumitem}
\usepackage{graphicx}
\usepackage{float}
\usepackage{standalone}
\usepackage{lipsum}
\usepackage{adjustbox}
\usepackage{verbatim}
\usepackage{mathabx}
\usepackage{booktabs} 
\usepackage{latexsym,eucal,color}
\usepackage{siunitx}
 
\newtheorem{thm}{Theorem}
\newtheorem{prop}{Proposition}
\newtheorem{cor}{Corollary}

\usepackage{physics}
\usepackage{xcolor}
\usepackage{rotating}
\usepackage{wrapfig}
\usepackage{caption}
\DeclareMathOperator*{\argmax}{arg\,max}
\DeclareMathOperator*{\argmin}{arg\,min}

\makeatletter
\let\Ginclude@graphics\@org@Ginclude@graphics
\makeatother

\title[Deep Autoencoders: From Understanding to Generalization Guarantees]{Deep Autoencoders: \\  From Understanding to Generalization Guarantees}
\usepackage{times}

\msmlauthor{%
 \Name{Romain Cosentino} \Email{rom.cosentino@gmail.com}\\
 \addr Rice University
 \AND
 \Name{Randall Balestriero} \Email{randallbalestriero@gmail.com}\\
 \addr Rice University%
  \AND
 \Name{Richard Baraniuk} \Email{richb@rice.edu}\\
 \addr Rice University%
  \AND
 \Name{Behnaam Aazhang} \Email{aaz@rice.edu}\\
 \addr Rice University%
}

\begin{document}

\maketitle

\begin{abstract}
A big mystery in deep learning continues to be the ability of methods to generalize when the number of model parameters is larger than the number of training examples. In this work, we take a step towards a better understanding of the underlying phenomena of Deep Autoencoders (AEs), a mainstream deep learning solution for learning compressed, interpretable, and structured data representations. 
In particular, we interpret how AEs approximate the data manifold by exploiting their continuous piecewise affine structure. 
Our reformulation of AEs provides new insights into their mapping, reconstruction guarantees, as well as an interpretation of commonly used regularization techniques. We leverage these findings to derive two new regularizations that enable AEs to capture the inherent symmetry in the data. Our regularizations leverage recent advances in the group of transformation learning to enable AEs to better approximate the data manifold without explicitly defining the group underlying the manifold. Under the assumption that the symmetry of the data can be explained by a Lie group, we prove that the regularizations ensure the generalization of the corresponding AEs.
A range of experimental evaluations demonstrate that our methods outperform other state-of-the-art regularization techniques.
\end{abstract}


\section{Introduction}
Autoencoders (AEs) provide a rich and versatile framework that discovers the data's salient features in an unsupervised manner. They are commonly leveraged to efficiently perform compression \citep{cheng2018deep}, denoising \citep{eraslan2019single}, data completion \citep{tran2017missing}, as well as pre-training supervised DNs \citep{erhan2010does}. Solving these tasks is equivalent to discovering the data's underlying manifold, a task becoming challenging in the high dimensional and the finite samples regime \citep{tenenbaum2000global,wang2014generalized,wang2016auto,facco2017estimating}.
To overcome these challenges and improve the efficiency of AEs, various explicit or implicit regularizations have been proposed \citep{vincent2008extracting,rifai2011contractive,makhzani2013k,falorsi2018explorations}. Despite these improvements, the underlying mechanisms and generalization capability of AEs are still poorly understood \citep{li2018on,pmlr-v89-nguyen19a,lei2020geometric}. 
\\
\\
\indent
A compelling approach to understanding the inner mechanisms of DNs considers their capability at modeling the ubiquitous symmetries in the
data \citep{paul2014does,mallat2016understanding}. Theoretically grounded models such as the Deep Scattering Network and its derivatives have been derived in accordance with this principle \citep{bruna2013invariant,anden2014deep,chen2014unsupervised,joint}.
In \citet{cohen2016group,DBLP:journals/corr/abs-1801-10130,kondor2018generalization} they propose to explain the success of deep convolutional architectures through the development of a theory of equivariant DNs; in particular, they provide $(i)$ an understanding and formalism behind the equivariance properties of DNs as well as their generalization, and $(ii)$ reduce the sample complexity of DNs by exploiting well-known symmetry group inherent to the image manifold.
\\
\\
\indent
Besides explicitly imposing specific group of transformations, the studies of DNs through that lens mainly consider the properties of internal layers of DNs, e.g., {\em convolution, pooling, per-layer representation}. In this paper, we propose a global analysis by considering the DN from a geometrical standpoint. By global analysis, we consider the understanding of the output of a DN given its input in an end-to-end manner. Such analysis is presently performed by leveraging the analytical continuous piecewise affine (CPA) map formulation of DNs, as described in \citet{balestriero2018mad}. Such an approach has two significant advantages; it is agnostic of the architecture, e.g., type of layer, nonlinearities, number of layers, and it provides an analytical formula for the entire network mapping. These criteria are crucial since the understanding of AEs performed in this work has the goal of developing practical tools that are not tied to any specific AE architecture.
\\
\\
\indent
In the present work, the CPA formulation is leveraged to take a step into answering the following questions: $(i)$ How an AE can effectively approximate the data manifold? $(ii)$ How can one improve and guarantee the generalization of AEs exploiting the symmetry in the data?
\\
\\
\indent
We will execute this by considering the following two-fold approach:
First, we provide an analytical and interpretable formulation of the CPA representation of the manifold spanned by AEs. We make explicit some critical properties of AEs such as what type of function do they belong to, how standard regularization techniques affect the AE mapping, and how the encoder and decoder per region affine mappings are related. 
Second, we exploit these insights to provide novel regularizations for AEs to capture the symmetry in the data. In particular, our regularizations constrain the global continuous piecewise affine surface spanned by AEs such that they adapt to the geometry of the data manifold modeled from as the orbit of a Lie group. We show that these regularizations constrain the entire surface even at locations in the manifold where data are missing, which is critical for the generalization of AEs. Besides, we show that these regularizations lead to generalization guarantees in the finite data regime.

Our contributions can be summarized as follows:
\begin{itemize}[leftmargin=*]
\setlength\itemsep{.004em}
    \item We highlight the fact that AEs provide a CPA approximation of the data manifold. From this analytical characterization, we interpret the role of the encoder, decoder, layer parameters, and latent dimension (Sec.~\ref{sec:AE_CPA}) as well as propose reconstruction guarantees (Sec.~\ref{sec:AE_RECONS}).
    
    \item Following these findings, we obtain interpretable formulas for the Jacobian and approximated Hessian of AEs (Sec.\ref{sec:tang-hess}); and leverage them to provide insights into standard regularization techniques employed in AEs (Sec.~\ref{sec:interp_usual}). 
    
    \item We demonstrate that when considering the symmetry of the data, we can impose constraints on an interpolation function, e.g., an AE, such that it approximates the data manifold driven by a Lie group (Sec.~\ref{sec:secondorderreg}). We turn these constraints into regularizations adapted to AEs and demonstrate their generalization guarantees under a finite data regime (Sec.~\ref{sec:secondorderAE}).
    
    \item We finally provide experimental validations of the developed regularizations which compete with state-of-the-art methods on various datasets (Sec.~\ref{sec:ex}).
\end{itemize}

\section{Background and Notations}
This section provides the required notions regarding the three frameworks this work intertwines.
\label{maso_back}

\subsection{Autoencoder}
An Autoencoder (AE) aims at learning an identity mapping, also known as auto-association \cite{ackley1985learning}, on a given dataset with a bottleneck latent dimension. It has been implemented first for image compression \cite{Cottrell}, speech recognition \cite{elman1988learning}, and dimensionality reduction \cite{baldi1989neural}. It is composed of two nonlinear maps: an encoder, denoted by $\boldsymbol{E}$ and a decoder, denoted by $\boldsymbol{D}$. The encoder maps an input $x \in \mathbb{R}^d$ to a hidden layer of dimension $h<d$, $\boldsymbol{E}(\boldsymbol{x})$, which encodes the salient features in the data \cite{Goodfellow-et-al-2016} and defines its \textit{code} or \text{embedding}. The decoder reconstructs the input from its \textit{code}, thus the entire AE map is defined as $\left (\boldsymbol{D} \circ \boldsymbol{E} \right) (\boldsymbol{x})$ with $\circ$ denoting the composition operator.

The weights of the AE are learned based on some flavors of reconstruction losses, e.g., the mean-square error for real data and the binary cross-entropy for binary data, between the output, $\left (\boldsymbol{D} \circ \boldsymbol{E} \right )(\boldsymbol{x})$, and the input, $\boldsymbol{x}$. To improve generalization, some regularizations can complement the reconstruction loss \cite{srivastava2014dropout} such as favoring sparsity of the \textit{code} \cite{makhzani2013k} or sparsity of the weights \cite{jarrett2009best}. Other types of regularization include injecting noise in the input leading to Denoising AE known to increase the robustness to small input perturbations \cite{vincent2008extracting}. Closer to our work, \cite{rifai2011higher} and \cite{rifai2011contractive} proposed to improve the robustness of the \textit{code} to small input perturbations by penalizing the curvature of the encoder mapping by regularizing the Jacobian as well as the Hessian of $\boldsymbol{E}$.

\subsection{Max Affine Spline Network}
\label{sec:maso}
A DN is an operator $\boldsymbol{f}_\Theta$ with parameters $\Theta$ composing $L$ intermediate {\em layer} mappings $\boldsymbol{f}_{\ell}$, $\ell=1,\dots,L$, that combine affine and simple nonlinear operators such as the {\em fully connected operator}, {\em convolution operator}, {\em activation operator} (applying a scalar nonlinearity such as the ubiquitous ReLU), or {\em pooling operator}. 

A DN employing nonlinearities such as (leaky-)ReLU, absolute value, and max-pooling is a continuous piecewise linear operator and thus lives on a partition $\Omega$ of the input space. As such, the DN's CPA mapping of an input $\boldsymbol{x}$ can be written as 

\begin{equation}
\label{eq:cpa}
    \boldsymbol{f}_{\Theta}(\boldsymbol{x}) = \sum_{\omega \in \Omega} 1_{\{\boldsymbol{x} \in \omega\}}\left(A_{\omega}\boldsymbol{x}+B_{\omega} \right)
\end{equation}
where $1$ defines the indicator function, $A_{\omega}$ and $B_{\omega}$ the per region affine parameters involving the DN per layer affine parameters , $W^{\ell}, \boldsymbol{b}^{\ell} \in \Theta, \forall \ell$, and the nonlinearities state of the region $\omega \in \Omega$ \citep{pmlr-v80-balestriero18b}.
The unit and layer input space partitioning can be rewritten as Power Diagrams, a generalization of Voronoi Diagrams \citep{balestriero2019geometry}; composing layers produce a Power Diagram subdivision.

\subsection{Lie Group Transformation}
\label{sec:learning_group}
The approaches mentioned in the introduction aiming at enhancing the capability of DNs to take into account the symmetries in the data consider only those that are explicitly modeled, s.a., translation, another line of work consists in learning them from the data.
In \citet{rao1999learning,sohl2010unsupervised,hashimoto2017unsupervised,bahroun2019similarity,pmlr-v119-cosentino20a}, they propose methods capable of discovering the symmetry within the data alleviating the need for explicitly defining appropriate equivalence classes for the data. In fact, in a simple computer vision dataset such as MNIST or in a music retrieval dataset such as GTZAN, there is more than translation and rotation to characterize efficiently the data \citep{ick2020learning}. This paragraph is dedicated to the understanding of such approximation methods, which will be an important part of our regularizations.

The approximation of Lie groups has been introduced by \cite{rao1999learning} and later extended in \cite{sohl2010unsupervised,wang2011lie}, and aims at learning the transformation operator underlying the data with the assumption that the dataset is the result of the action of a group on a sample. This framework has an essential place in neuroscience as there is evidence of an underlying network of neurons enabling the detection of a class of equivalence via transformation learning \citep{kriegeskorte2008matching,sengupta2018manifold,bahroun2019similarity}.

In the case of a Lie group, the dataset can be modeled according to the first-order Lie equation
\begin{equation}
\label{lie_equation}
\frac{d \boldsymbol{x}(\theta)}{d \theta} =  G \boldsymbol{x}(\theta),
\end{equation}
where $\boldsymbol{x}(\theta)\in \mathbb{R}^{d}$, $\theta$ is the coefficient governing the amount of transformation, and $G \in \mathbb{R}^{d \times d}$.
This first-order differential equation indicates that the variation of the data is linear with respect to the data and depends on the infinitesimal operator $G \in \mathcal{T}_{I}\mathcal{G}$ where $\mathcal{T}_{I}\mathcal{G}$ denotes the Lie algebra of the group $\mathcal{G}$, i.e., the tangent of the group at the identity element. An introduction to group transformations can be found in \cite{hall2015lie}. The solution of Eq.~\ref{lie_equation} is given by $\boldsymbol{x}(\theta) = \exp (\theta G) \boldsymbol{x}(0)$, a simple example with the group of rotation is given in Appendix~\ref{Lieexample}.

While the learnability of the exponential map is tedious, one can exploit its Taylor series expansion to learn the infinitesimal operator. In fact, for a small $\epsilon$ we have
\begin{align}
\label{eq:approx_lie_first}
    \boldsymbol{x}(\theta +\epsilon)  \approx (I+\epsilon G) \boldsymbol{x}(\theta) 
\end{align}  
The operator $G$ can thus be learned using data that are close to each other as they result from small transformations and thus follow this approximation.
Without this form of supervision, the search for neighbor data is achieved by the nearest neighbor algorithm, as in \cite{hashimoto2017unsupervised}. Note that in our case, we will consider multiple transformations, each parametrized by a $1$-dimensional Lie group, i.e. $\textbf{x}(\theta) = \prod_{k=1}^h \exp(\theta_k G_k) \textbf{x}(0)$, where $\theta \in \mathbb{R}^h$. In that case the first order approximation around the identity element of each group, as Eq.~\ref{eq:approx_lie_first}, becomes $ \boldsymbol{x}(\theta +\epsilon)  \approx (I+\sum_{k=1}^{h}\epsilon_k G_k) \boldsymbol{x}(\theta) $, where $\epsilon \in \mathbb{R}^h$ and with $\epsilon_k$ being the transformation parameter associated to infinitesimal operator $G_k$.



\section{Max Affine Spline Autoencoders}
We now leverage the CPA operator defined in Eq.~\ref{eq:cpa} to reformulate AEs and propose to interpret the role of the encoder and decoder (Sec.~\ref{sec:AE_CPA}), derive a necessary condition for the reconstruction of piecewise linear data surface (Sec.~\ref{sec:AE_RECONS}), characterize its per region surface via the Jacobian and approximated Hessian of the CPA operator (Sec.~\ref{sec:tang-hess}), and finally use these findings to analyze commonly used variations of AEs (Sec.~\ref{sec:interp_usual}).

\label{sec:interp}

\begin{figure}[t]
\centering
\begin{minipage}{.45\textwidth}
    \centering
    \includegraphics[width=1\linewidth]{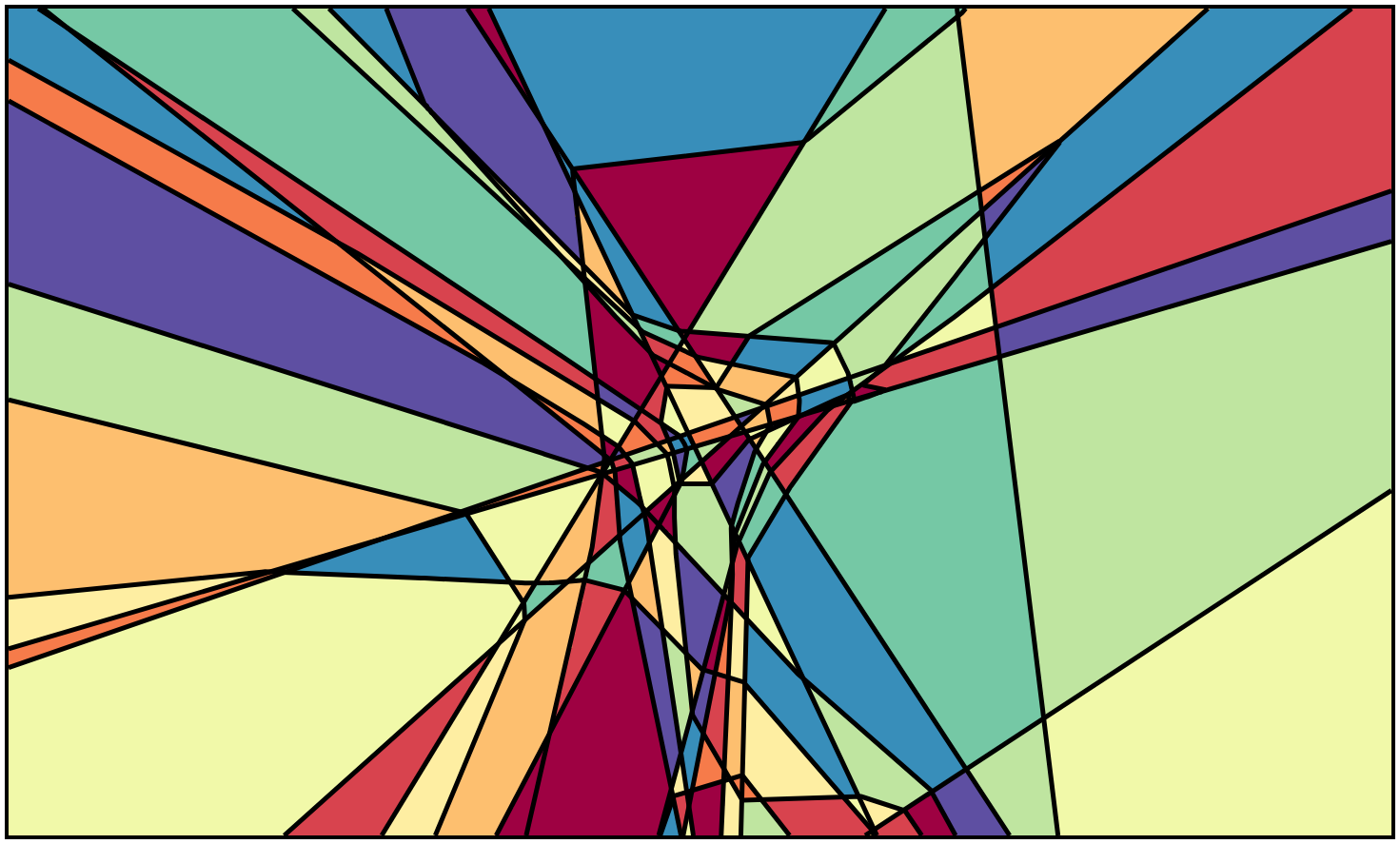}
\end{minipage}
\hspace{1cm}
\begin{minipage}{.45\textwidth}
    \centering
    \includegraphics[width=1\linewidth]{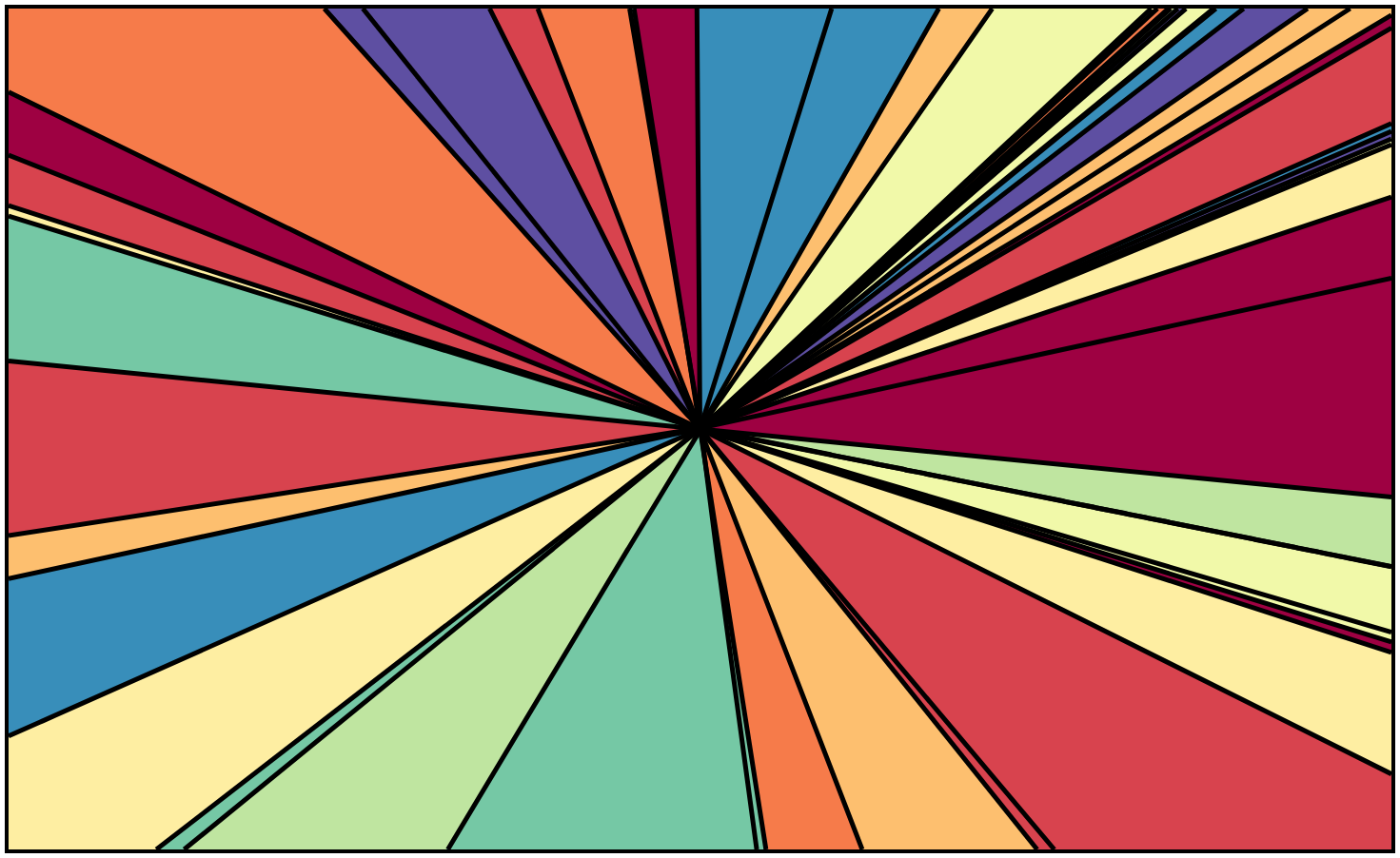}
\end{minipage}
    \caption{2-dimensional visualizations of the input space partitioning $\Omega^{E,D}$ induced by two randomly initialized AEs with bias (\textit{left}) and zero bias (\textit{right}). 
    Each region, depicted by a particular color,  bounded by the black lines has a set of CPA parameters $A^{E}_{\omega}, A^{D}_{\omega}, B^{E}_{\omega}, B^{D}_{\omega}$ described in Eq.~\ref{masoparams} which depend on the per-layer affine parameters as well as the state of the nonlinearities of the region $\omega$. To reconstruct its input, an AE achieves an affine map for each region; its output for a sample of a given region $\omega$ is provided by Eq.~\ref{eq:masoAE}.}
    \label{fig:inputspacepartitioning}
\end{figure}

The output of a CPA DN is formed as per Eq.~\ref{eq:cpa}. An AE composing two CPA functions, the {\em encoder} and the {\em decoder}, the entire mapping remains a CPA with an input space partition and per region affine mappings. Because we can consider an AE as a network or as the composition of two networks, we will consider two different space partitioning. The partition of the input, i.e., data space, induced by the entire AE, and denoted by $\Omega^{E,D}$, as well as the partition of the decoder induced in the latent space, i.e., bottleneck layer, and denoted by $\Omega^D$. Examples of the entire AE partitioning, i.e., $\Omega^{E,D}$, can be visualized in Fig.~\ref{fig:inputspacepartitioning}.

\subsection{Autoencoders as a Continuous Piecewise Affine Mapping}
\label{sec:AE_CPA}
Now, let $\omega \in \Omega^{E,D}$ defines a region induced by the AE partitioning in the input space as described in Sec.~\ref{sec:maso}. 
Given a d-dimensional sample $\boldsymbol{x} \in \omega$, the max affine spline formulation of the AE mapping is defined as
\begin{equation}
\label{eq:masoAE}
 \boldsymbol{D} \circ \boldsymbol{E}(\boldsymbol{x}) = A^{D}_{\omega}A^{E}_{\omega}\boldsymbol{x} + A^{D}_{\omega}B^{E}_{\omega} + B^{D}_{\omega}, 
\end{equation}
where $\circ$ is the composition operator, $A^{D}_{\omega} \in \mathbb{R}^{d \times h}, A^{E}_{\omega} \in \mathbb{R}^{h \times d}, B^{E}_{\omega} \in \mathbb{R}^{h}$ and $B^{D}_{\omega} \in \mathbb{R}^d$ with $d$ being the dimension of the input space and $h$ the bottleneck dimension. 

The mapping from these global parameters to the per-layer ones is performed as follows.
First, we denote by $W^{\ell} \in \mathbb{R}^{d_{\ell} \times d_{\ell-1}},\boldsymbol{b}^{\ell} \in \mathbb{R}^{d_{\ell}}$ the affine parameters of each layer, where $\ell \in \left \{1,\dots,L\right \}$ defines the encoder indexes and $\ell \in \left \{L+1,\dots,L+P \right \}$ the decoder ones (with structure depending on the layer type), where $L$ denotes the number of encoder layers, $P$ the number of decoder layers, $d_{\ell-1}$ the input dimension of the layer $\ell$ and $d_{\ell}$ its output dimension. We have that $d_{L}=h$ the bottleneck dimension, $d_{0}=d_{L+P}=d$ the input and output dimension. Then, we also denote by $Q^{\ell}$ the diagonal matrices encoding the region induced states of the nonlinearities, $(0,1)$ for ReLU, $(-1,1)$ for absolute value. Finally, the parameters of the max affine spline AE formulation described in Eq.~\ref{eq:masoAE} are defined as
\begin{align}
\label{masoparams}
    A^{E}_{\omega} =& W^{L}Q_{\omega}^{L-1}W^{L-1}\dots Q_{\omega}^{1}W^{1}\;\;\text{ and }\;\;
    B^E_{\omega} = \boldsymbol{b}^{L}+\sum_{i=1}^{L-1}W^{L}Q_{\omega}^{L-1}W^{L-1}\dots Q_{\omega}^{i}\boldsymbol{b}^{i}.
\end{align}
$A^{D}_{\omega}$ and $A^{D}_{\omega}$ are defined similarly with $\ell \in \left \{L+1,\dots,L+P \right \}$.
Therefore, there is a direct mapping from the intuitive piecewise affine parameterization of the network to the per-layer parametrization as it is commonly used in the literature. 

Given these analytical maps, we now provide insights into the AE approximation. Let's rewrite Eq.~\ref{eq:masoAE} as
\begin{align}
\label{eq:proj}
\boldsymbol{D} \circ \boldsymbol{E}(\boldsymbol{x}) & = \sum_{k=1}^{h} \left \langle  \textbf{a}^{E^{T}}_{k}[\omega],\boldsymbol{x} \right \rangle \textbf{a}^{D}_{k}[\omega] + B^{E,D}_{\omega}  = A^{D}_{\omega} \boldsymbol{\mu}_{\boldsymbol{x}} + B^{E,D}_{\omega},
\end{align}
where $B^{E,D}_{\omega} = A^{D}_{\omega}B^{E}_{\omega}+B^{D}_{\omega}$, $\textbf{a}^{E^T}_{k}[\omega]$ are the rows of $A^{E}_{\omega}$, $\textbf{a}^{D}_{k}[\omega]$ are the columns of $A^{D}_{\omega}$. This is the shifted mapping of $\boldsymbol{x}$ onto the subspace spanned by $A^{D}_{\omega}$ and with coordinates driven by $A^{E}_{\omega}$. 

From Eq.~\ref{eq:proj}, we deduce the per region role of the encoder and decoder. The samples of each region $\omega \in \Omega^{E,D}$, are expressed in the basis defined by the decoder region-dependent parameter $A^{D}_{\omega}$, i.e., the per region parametric representation of the approximated manifold, and the coordinates of this sample in such a basis are induced by the region-dependent parameter $A^{E}_{\omega}$, the whole mapping is then shifted according to both the encoder and decoder CPA parameters.

\subsection{Reconstruction Guarantees}
\label{sec:AE_RECONS}
We now derive a necessary condition on the CPA parameters, $A^{D}_{\omega}, A^{E}_{\omega}$, such that the AE achieves perfect reconstruction on a given continuous piecewise linear surface in the case of zero bias as often used in practice \cite{mohan2020robust}.
\begin{prop}
\label{prop_2}
A necessary condition for the zero-bias AE to reconstruct a continuous piecewise linear data surface is to be bi-orthogonal as per
$
    \forall \boldsymbol{x} \in \omega, \;\; \boldsymbol{D} \circ \boldsymbol{E}(\boldsymbol{x}) = \boldsymbol{x} \implies \left \langle \boldsymbol{a}^{D}_k[\omega], \boldsymbol{a}^{E}_{k'}[\omega] \right \rangle = 1_{ \left \{k=k' \right \}}.
$
(Proof in Appendix~\ref{proof:prop2}.)
\end{prop}

That is, if a continuous piecewise linear surface is correctly approximated, we know that the parameters of the MAS operator describing the encoder and decoder will be bi-orthogonal, i.e., the column vectors of $A^{D}_{\omega}$ and the row vectors of $A^{E}_{\omega}$ form a bi-orthogonal basis.

We now propose to give intuitions regarding this condition by utilizing the mapping between CPA parameters and layer weights as per Eq.~\ref{masoparams}. In fact, the following corollary provides the conditions for the bi-orthogonality to be fulfilled depending on the weights of the autoencoder, i.e., $W^{\ell}$. For the sake of clarity, we consider the case of a $2$-layer ReLU AE. 
\begin{cor}
\label{cor:corrolary_reconst}
Let $\boldsymbol{E}$ and $\boldsymbol{D}$ be a $2$-layer ReLU network with respective weights $W^{1} \in \mathbb{R}^{h \times n} $ and $W^{2}\in \mathbb{R}^{n \times h}$, as per Eq.~\ref{masoparams}. We denote by $W^1_{i,j}$ the $i^{th}$ row and $j^{th}$ column of the weight matrix $W^{1}$. Now, $\forall x \in \mathcal{X}$, a necessary condition for bi-orthogonality is that, for each $k,k' \in \left \{1,\dots,h \right \}$, one of the following is fulfilled:
\begin{enumerate}[label=(\roman*)]
  \setlength{\itemsep}{0pt}
    \item $W^{1^T}_{k',.} x \leq 0. $
    \item $\forall i \in \left \{1,\dots,d \right \}, \;\;\; W^{2^T}_{i,.} \boldsymbol{E}(x) \leq 0. $
    \item $\forall i \in \left \{1,\dots,d \right \}, \;\;\; W^{2^T}_{i,.} \boldsymbol{E}(x) >0$ and $\left \langle W^{2}_{.,k}, W^{1}_{k',.} \right \rangle = 0.$
    \item $\sum_{i=1}^{d} W^{2}_{i,k} W^{1}_{k',i} 1_{ \left \{ W^{2^T}_{i,.} \boldsymbol{E}(x) >0 \right \}}=0.$
\end{enumerate}
(Proof in Appendix~\ref{proof:corrolary1}.)
\end{cor}
This results shows that the bi-orthogonality condition can be obtained via a combination of orthogonality conditions between the weights and/or nonlinearity activations. 

For instance, the proposition $(i)$ corresponds to the case where the input of the $k'$ unit in the bottleneck layer is negative, condition $(ii)$ is the case where the input of all output units is negative, condition $(iii)$ corresponds to a linear decoder and orthogonality of the weights, and $(iv)$ corresponds to an orthogonality condition between the $k^{th}$ column of the decoder weight with the $k'^{th}$ row of the encoder weight modulo the activations of the decoder layer. Note that if $(ii)$ and $(iii)$ hold for multiple regions $\omega \in \Omega^{E,D}$ it implies that the decoder is linear with respect to the coordinate space and forms a linear manifold. Thus, these are not realistic conditions to have efficient AEs.

\begin{figure}
    \centering
    \hspace{1.2cm}
    \begin{minipage}{0.45\linewidth}
    \includegraphics[width=.8\linewidth]{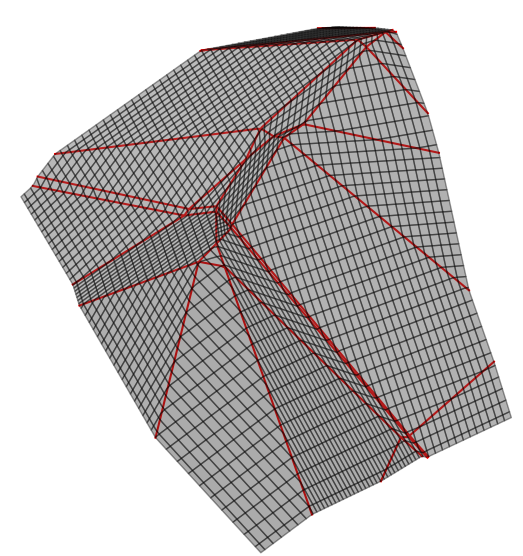}
    \end{minipage}
    \begin{minipage}{0.45\linewidth}
    \includegraphics[width=.8\linewidth]{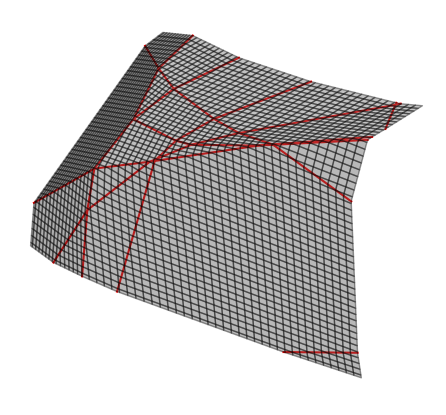}
    \end{minipage}
    \caption{Piecewise linear surfaces induced by two randomly initialized AE decoders and visualized in the ambient space of dimension $d=3$ (latent dimension being $h=2$). The gray denotes the regions, and the red lines their borders. As they correspond to the MAS surface induced by the decoder, each gray region has a slope characterized by the Jacobien of the decoder as in Eq.~\ref{eq:jaco_D}. Our work aims at developing a constraint on these surfaces via their per region tangent, such that they approximate the manifold defined by the orbit of a signal with respect to the action of a group.}
    \label{fig:surface}
\end{figure}
 
\subsection{Tangents and Hessian}
\label{sec:tang-hess}
From the CPA formulation, we observed that for each region $\omega \in \Omega^{E,D}$, $\boldsymbol{D} \circ \boldsymbol{E}$ defines a composition of two continuous piecewise affine functions, each defined respectively by the parameters $A^{E}_{\omega}$, $ B^E_{\omega}$, and $A^{D}_{\omega}$, $ B^D_{\omega}$. We can thus derive simple analytical formulas for the per region Jacobian and approximated Hessian of the AE.

The Jacobian of the AE for a given region $\omega \in \Omega^{E,D}$ is given by
\begin{align}
\label{eq:jaco}
 J_{\omega}[\boldsymbol{D} \circ \boldsymbol{E}] = A^{D}_{\omega} A^{E}_{\omega}.
\end{align}
More details regarding the Jacobian are given in Appendix~\ref{ap:jaco}.
It is also clear that the rank of the Jacobian is upper bounded by the latent dimension as $\text{rank}(J_{\omega}[\boldsymbol{D} \circ \boldsymbol{E}]) \leq h$, where $h$ is the number of units of the bottleneck layer of the AE, and in general by the $\min_{\ell} d_{\ell}$.
This dimension is directly related to the manifold's dimension that one aims to approximate, assuming that all other layer widths are larger than $h$.

One can similarly obtain the per region tangent of the decoder, as it defines the per region parametric representation of the manifold, see Fig.~\ref{fig:surface}. We recall that we denote by $\Omega^{D}$ the partition of the latent space induced by the decoder
\begin{align}
\label{eq:jaco_D}
 \forall \omega \in \Omega^D, J_{\omega}[\boldsymbol{D}] = A^{D}_{\omega},
\end{align}
where the columns of $A^{D}_{\omega}$ form the basis of the tangent space induced by $\boldsymbol{D}$. 

The characterization of the curvature of the approximation of the data manifold can be done using the per region Hessian defined by $H_{\omega}, \forall \omega \in \Omega^{D}$, which in our case will be defined as the sum of the difference of neighboring tangent planes.
\begin{equation}
    \label{eq:Hessian}
    \forall \omega \in \Omega^D, \left \| H_{\omega} \right \|_{F} = \sum_{\omega' \in \mathcal{N}(\omega)}   \left \| J_{\omega}[\boldsymbol{D}] - J_{\omega'}[\boldsymbol{D}] \right \|_F,
\end{equation}
where $\mathcal{N}(\omega)$ denotes the set of neighbors of region $\omega$ and $\left \| . \right \|_{F}$ is the Frobenius norm. This approach is based on the derivation described in \citet{rifai2011higher}. In practice, we use a stochastic approximation of the sum by generating a small mini-batch of a few corrupted samples which induce neighboring regions.

\subsection{Interpretability of Regularization Techniques}
\label{sec:interp_usual}
We are now interested in leveraging these findings to analyze and interpret common AE regularizations.
\begin{enumerate}[label=(\roman*),leftmargin=*]
\item \textbf{Higher-Order Contractive AE} \citep{rifai2011higher}: This regularization penalizes the energy of the first and approximated second derivative the encoder map for any region containing a training sample, i.e., $\left \|A^{E}_{\omega} \right \|_{F}$ and $ \sum_{\omega' \in \mathcal{N}(\omega)} \left \|A^{E}_{\omega}- A^{E}_{\omega'} \right \|_{F}$. In the case of a ReLU AE, we know from Eq.~\ref{eq:masoAE} and the submultiplicativity of the Frobenius norm that the norm of the Jacobian is upper-bounded by $\left \|W^{L} \right \|_{F} \times \dots \times \left \| W^{1} \right \|_{F}$. Therefore adding a weight-decay penalty on the encoder weights induces the first-order contractive AE. The second-order induces the curvature of the piecewise linear map $A^{E}$ to be small. Note that it is the per-region affine map induced by the encoder that is regularized, and that it depends on the region's activation codes, i.e., $ Q^{i} $ and $W^{i} \; \; \forall i \in \left \{1,\dots,L \right \}$. Thus,  if two neighboring regions have only have few changes in their code, and that the associated weights are small, then, such a constraint does not affect the overall curvature. On the other hand, if between two regions, the code of a unit having a weight with large amplitude does not change, then the regularization does not affect the curvature either, see details with a toy example in Appendix~\ref{sec:hoc}.

\item \textbf{Denoising AE} \citep{vincent2008extracting}:  Denoising AE is known to have a similar effect than the weight-decay penalty on the DN architecture \cite{wager2013dropout}. A penalty on the energy of $W^{\ell}$ induces a penalty on the energy of the $A^{E}_{\omega}$ and $A^{D}_{\omega}$ $,\forall \omega \in \Omega^{E,D}$. Therefore, it constrains each piece's slope to be as flat a possible, implying that the piecewise linear map focuses on approximating the low-frequency content in the data, which reinforces the learning bias of deep networks towards low-frequency information \citep{rahaman2018spectral}. Thus, we see how denoising and Higher-Order Contractive are tied together.
\end{enumerate}
\begin{figure}[t!]
    \centering
    \begin{minipage}{0.03\linewidth}
    \rotatebox{90}{\# samples in ball}
    \end{minipage}
    \begin{minipage}{0.398\linewidth}
    \includegraphics[width=1\linewidth]{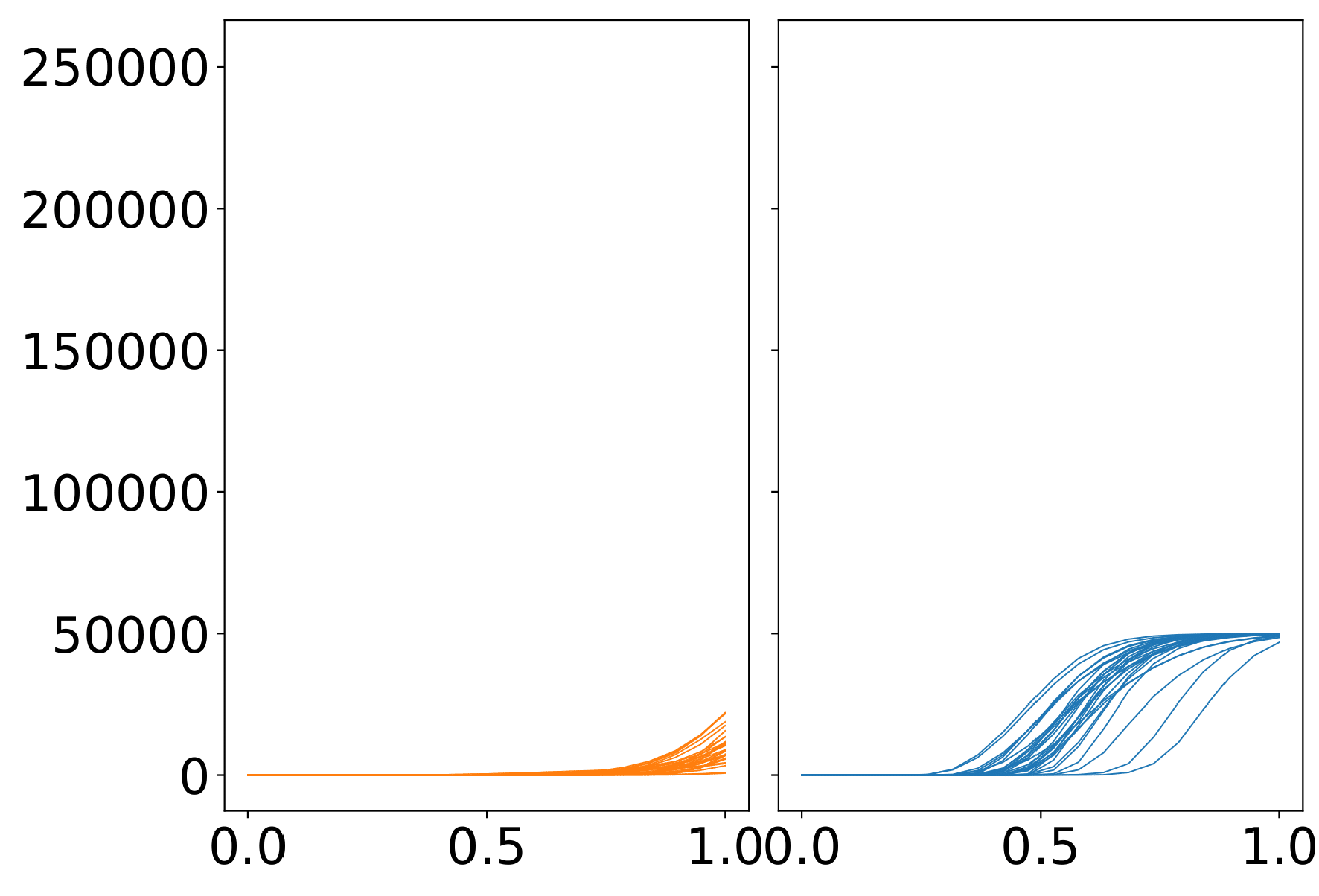}
    \end{minipage}
    \begin{minipage}{0.03\linewidth}
    \rotatebox{90}{\# regions in ball}
    \end{minipage}
    \begin{minipage}{0.503\linewidth}
    \includegraphics[width=1\linewidth]{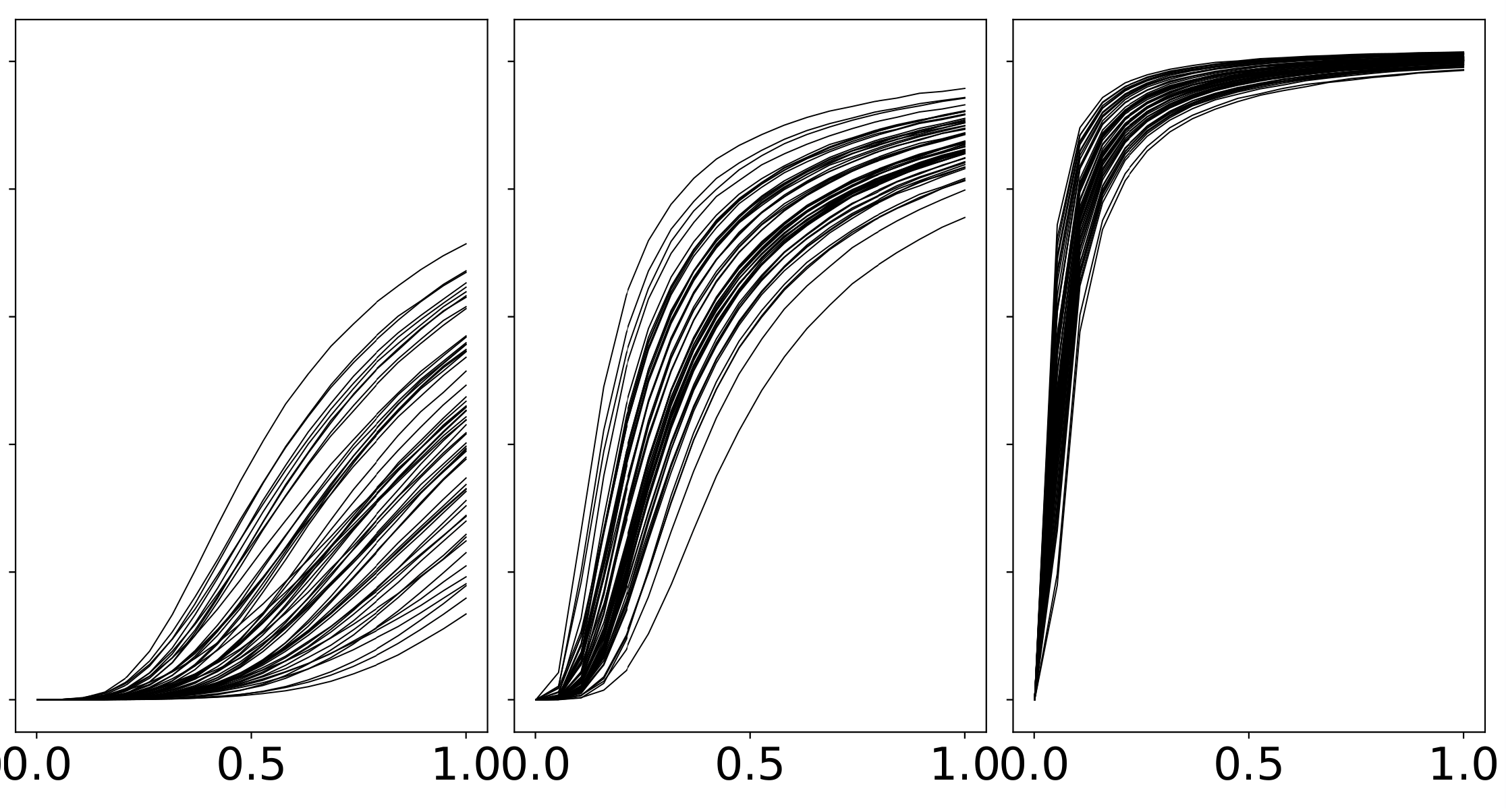}
    \end{minipage}
    \\
    \hspace{1.6cm}
    \begin{minipage}{0.17\linewidth}
    \centering
    radius
    \end{minipage}
    \begin{minipage}{0.17\linewidth}
    \centering
    radius
    \end{minipage}
    \begin{minipage}{0.17\linewidth}
    \centering
    radius
    \end{minipage}
    \begin{minipage}{0.17\linewidth}
    \centering
    radius
    \end{minipage}
    \begin{minipage}{0.17\linewidth}
    \centering
    radius
    \end{minipage}
    \caption{The first and second figures (from left to right) represent the number of data points inside a ball of growing radius (\textit{first} to \textit{second}: CIFAR10, MNIST). From the third to the last figure (from left to right), we show the number of regions in the latent space of the AE inside the same ball of growing radius for different AE architectures (\textit{third} to \textit{fifth}: Small MLP, Large MLP, Convolutional). 
    We observe that the number of regions induced by the AE partitioning of any DN architecture in any randomly sampled ball is much larger than the number of data for any radius.} 
        \label{fig:number_regions}
\end{figure}
\indent
Now that we understand autoencoders' different components and their underlying functionality, we propose to constrain the surface's geometry spanned by the CPA map.
In fact, we can see in Fig.~\ref{fig:number_regions} that for a given ball positioned in the input data space, the number of regions induced by the AE is much larger than the number of data. It is then clear that only a few of the regions contain data points. Thus, besides the implicit constraints of Deep Network, such as weight sharing on convolutional nets, and the continuity constraints of the mapping, there are no other structural constraints on the behavior of regions where no training data are available \citep{gamba2020hyperplane,ergen2020convex}. {\em  There is, therefore, a need to constrain all the regions of the CPA to guarantee the generalization capability of AEs.} The next sections aim at deriving a regularization that imposes a global structure on this mapping, which is effective even where no data points are available.

\section{Generalization Via Lie Group Regularizations}
For the remaining of the paper, we model the dataset as the orbit of a Lie group, that is, as per Eq.~\ref{lie_equation},  $\boldsymbol{x}(\theta) = \exp(\theta G)\boldsymbol{x}(0),  \theta \in \mathbb{R}, G \in \mathcal{T}_{I}\mathcal{G}$, where $\mathcal{T}_{I}\mathcal{G}$ denotes the Lie algebra of the group $\mathcal{G}$. We also assume that $\forall \theta \in \mathbb{R}$, $x(\theta) \in \mathbb{R}^d \setminus \left \{0 \right \}$ to avoid degenerated cases. Our aim is to provide a regularization that leads to generalization guarantees, i.e., the AE is equal to $\textbf{x}$ at any location of the manifold. In Sec.~\ref{sec:secondorderreg}, we first provide such a regularization from a general point of view, that is, we consider the approximation of $\textbf{x}$ by a smooth interpolation function ($C^2(\mathbb{R},\mathbb{R}^d)$). We then translate this condition for CPA operators (Sec.~\ref{sec:secondorderAE}) to apply it to any AE. We then demonstrate the generalization guarantees it yields (Sec.~\ref{sec:guar}).

\subsection{Approximation of Lie Group's Orbit}
\label{sec:secondorderreg}
First, we want to understand under which condition a interpolation function $\boldsymbol{f} \in C^{2}(\mathbb{R}, \mathbb{R}^d)$ coincides with the orbit of $\boldsymbol{x}(0) \in \mathbb{R}^d \setminus \left \{0 \right \}$ under the action of the group $\mathcal{G}$.
In particular, we propose to exploit a regularization that induces an orbit of a Lie group, such as
\begin{equation}
\label{eq:reg}
     \mathcal{R}_{k}(\boldsymbol{f}) \triangleq \int \left \| \frac{d^{k}\boldsymbol{f}(\theta)}{d \theta^{k}} -G\frac{d^{k-1}\boldsymbol{f}(\theta)}{d \theta^{k-1}} \right \| d\theta, 
\end{equation}
where $\frac{d^k\boldsymbol{f}(\theta)}{d \theta^k}$ denotes the $k^{th}$ order derivative of $\boldsymbol{f}$.


This regularization constrains $\boldsymbol{f}$ such that its $k^{th}$ order derivative is a linear map of the $k-1$ order. 
In the following theorem, we show that, for $k \in \left \{1,2 \right \}$, such regularization coupled with an interpolation loss function leads to a perfect approximation of the data manifold $x$. 
That is, $\boldsymbol{f}$ coincides with $\boldsymbol{x}(\theta)= \exp(\theta G)\boldsymbol{x}(0), \forall \theta \in \mathbb{R}$ if and only if $\frac{d^{k} \boldsymbol{f}(\theta)}{d \theta^{k}} =  G \frac{d^{k-1} \boldsymbol{f}(\theta)}{d \theta^{k-1}}$ and it exists a certain number of $\theta_i$, depending on the order $k$, such that $\boldsymbol{f}(\theta_i)=\boldsymbol{x}(\theta_i)$.
Note that the restriction to the first two orders is natural as we will apply these results on continuous piecewise affine maps, in which the second-order can only be approximated using stochastic approximation as per Sec.~\ref{sec:tang-hess}.

\begin{thm}
\label{prop:oneisenough}
For all $k \in \left \{1,2 \right \}$, assuming $G$ is invertible, and that a function $\boldsymbol{f}$ minimizes the regularization $\mathcal{R}_{k}(\boldsymbol{f})$ and it exists $\theta_i, i \in \left \{1,\dots,k\right \}$ such that $\boldsymbol{f}(\theta_i) = \boldsymbol{x}(\theta_i)$ then $\boldsymbol{f}$ has perfect generalization as in 
\begin{equation}
    \mathcal{R}_{k}(\boldsymbol{f}) = 0 \; \text{ and } \; \exists \theta_i \in \left \{1,\dots,k\right\} \; \text{ s.t. } \; \boldsymbol{f}(\theta_i)=\boldsymbol{x}(\theta_i) \iff  \forall \theta, \boldsymbol{x}(\theta)=  \boldsymbol{f}(\theta).
\end{equation}
(Proof in Appendix~\ref{proof:cor1}.)
\end{thm}

Thus an interpolant $\boldsymbol{f}$, can approximate the orbit of a Lie group, utilizing two components, the aforementioned regularization with $k \in \left \{1,2\right \}$, and a reconstruction error that force the interpolation function to coincide with $k$ training samples.

\subsection{Lie Group Regularizations for Autoencoders}
\label{sec:secondorderAE}
The derived regularizations were based on a smooth interpolant $\boldsymbol{f}$ and need to be adapted to the case of a CPA map. To do so, there are several crucial considerations:
\begin{enumerate}[label=(\roman*),leftmargin=*]
\item For the sake of clarity, the previous section illustrated the case of a one-dimensional group. Here we propose to generalize such an approach to multiple groups of transformations. We, therefore, consider the case of $h$ infinitesimal operators $G_1,\dots,G_h$ each corresponding to a $1$-dimensional group, as explained in Sec.~\ref{sec:learning_group}.
\item The second-order regularization requires constrains the Hessian of the CPA, which by definition, can only be approximated stochastically as explained in Sec.~\ref{sec:tang-hess}.
\item The assumption on the data is that they are generated by $h$ transformation groups. Thus, the intrinsic dimensionality of the data is at most $h$. Therefore, the size of the bottleneck layer, which corresponds to the maximum dimension of the manifold the autoencoder can generate (as per Sec.~\ref{sec:AE_CPA}), is also $h$.
\end{enumerate}

\paragraph{The case $k=1$:}
\noindent
The first-order regularization corresponds to the assumption that data that are generated by the decoder and that are close to each other result from small transformations of one to another. As per Eq.~\ref{eq:approx_lie_first}, we obtain
\setlength\itemsep{.004em}
\begin{align}
\label{eq:regAE_first}
    \mathcal{R}_{1}(\boldsymbol{D}) & \triangleq \min_{G_1,\dots,G_h}\int_{\mathbb{R}^h} \int_{\mathcal{N}(\theta)}  \min_{\epsilon_1, \dots,\epsilon_h} 
    \left\| \boldsymbol{D}(\theta) - (I+\sum_{k=1}^h \epsilon_k G_k) \boldsymbol{D}(\theta') \right \|_{2} d\theta' d\theta,  
\end{align}
where $\mathcal{N}(\theta)$ denotes the neighborhood of $\theta \in \mathbb{R}^h$, the parameters $\epsilon_1,\dots,\epsilon_h$ are the scalars corresponding to the scale of the transformations, and the $G_1,\dots,G_h$ the infinitesimal operators. The optimal parameters $\epsilon^{\star} = \left [\epsilon_1,\dots,\epsilon_h \right ]^T$ used during the training of the regularized AE are provided in Proposition~\ref{prop:theta_1} in Appendix~\ref{app:params}.
Note that the infinitesimal operators are learned using stochastic gradient descent and that the approximation of the integrals in Eq.~\ref{eq:regAE_first} is developed in Sec.~\ref{sec:exp}.

\paragraph{The case $k=2$:}
\noindent
We know from Sec.~\ref{sec:tang-hess} that for each region $\omega \in \Omega^D$, the decoder is characterized by its tangent plane, $A^{D}_{\omega}$. The second-order regularization imposes that each tangent plane of the AE is related to their neighboring tangents plane by small transformations. Again, considering the linearized exponential maps and exploiting the definition of the Hessian in Eq.~\ref{eq:Hessian} we obtain the following second-order regularization on the CPA
\begin{align}
\label{eq:regAE_second}
    \mathcal{R}_{2}(\boldsymbol{D}) & \triangleq \min_{G_1,\dots,G_h}\int_{\mathbb{R}^h} \int_{\mathcal{N}(\omega)}  \min_{\epsilon_1, \dots,\epsilon_h} 
    \left\| J_{\omega}[\boldsymbol{D}] - (I+\sum_{k=1}^h \epsilon_k G_k) J_{\omega'}[\boldsymbol{D}] \right \|_{F} d \omega' d\omega,  
\end{align}
where $\mathcal{N}(\omega)$ denotes the set of neighbors of region $\omega$ and $\left \| . \right \|_{F}$ is the Frobenius norm and we recall that $ J_{\omega'}[\boldsymbol{D}]= A^{D}_{\omega'} $ and $ J_{\omega}[\boldsymbol{D}] =  A^{D}_{\omega} $. The implementation regarding the sampling of neighboring regions is detailed Sec.~\ref{sec:exp}. In this case also, the optimal parameters $\epsilon^{\star} = \left [\epsilon_1,\dots,\epsilon_h \right ]^T$ used during the training of the regularized AE are provided in Proposition~\ref{prop:theta_2} in Appendix~\ref{app:params}.

Let us now provide interpretations regarding the Lie group regularizations we developed.
While the first-order regularization constrains the AE mapping, the second-order constrains the AE's tangent plane of each region.
In the first-order case the distance between $(I+\sum_{k=1}^h \epsilon_k G_k) \boldsymbol{D}(\theta')$, which corresponds to small transformations of the sample generated by the decoder, and $D(\theta)$ is minimized. Thus, such a regularization constrains the AE mapping to approximate the orbit induced by the infinitesimal generators. Then, the second-order regularization aims at minimizing the distance between $(I+\sum_{k=1}^h \epsilon_k G_k) J_{\omega'}[\boldsymbol{D}]$, which is the small transformation of the tangent plane of region $\omega'$, and $J_{\omega}[\boldsymbol{D}]$. This means that the second-order regularization constrains the Hessian of the decoder, which defines the angle between neighboring piecewise linear maps, to approximating the angle of the data manifold. Therefore, this penalization enforces the curvature of the piecewise linear map to fit the curvature of the orbit. Besides, as opposed to the Higher-Order Contractive AE \citep{rifai2011higher}, these regularizations constrain all the piecewise affine regions whether they contain training data or not as they do not rely on samples from the dataset. This is crucial to provide generalization guarantees in a finite data regime.

\paragraph{Remark:} Note that, because of the first order approximation around the identity of the group, the higher order term containing the information regarding the commutativity of the group are not taken into account in our approach. The non-abelian effects are in fact of the order $O(\epsilon^2)$. Therefore, in this work, no assumptions are made on the commutativity of the group.


\subsection{Approximation Error}
\label{sec:guar}
In Sec.~\ref{sec:secondorderreg}, we showed that if the regularization defined in Eq.~\ref{eq:reg} is equal to zero for any given $k \in \left \{1,2 \right \}$, and if the interpolation function $\boldsymbol{f}$ coincides with the data manifold defined by $\boldsymbol{x}$ on $k$ points, then $\boldsymbol{f}$ coincides with $\boldsymbol{x}$. We now derive the generalization guarantees in the particular case where $\boldsymbol{f}$ is a CPA approximant.

Based on the assumption that $(i)$ a region of the real manifold is correctly approximated, $(ii)$ one of the regularizations defined in Eq.~\ref{eq:regAE_first},\ref{eq:regAE_second} is minimized, and that $(iii)$ the infinitesimal operator $G$ obtained from the regularization coincides with the infinitesimal operator of the group governing the data, we obtain the following bound on the approximation of the data manifold.

\begin{thm}
\label{prop:tang}
If on a region $\omega' \in \Omega^{D}$ the matrix $A^{D}_{\omega'}$ forms a basis of the manifold tangent space on this region, and it exists $k \in \left \{1,2 \right \}$ such that $\mathcal{R}_{k}(\boldsymbol{D})=0$ then for all regions $\omega \in \Omega^{D}$ the basis vectors of $A^{D}_{\omega}$ are the basis vector of the tangent of the data manifold and the distance between the continuous piecewise affine map and the data manifold is upper bounded by the radius of the regions as per
\[ d \left ( \cup_{\omega \in \Omega^{D}} \mathcal{T}_{AE}(\omega), \mathcal{X} \right ) \leq  \sum_{\omega_i \in \Omega^D} \text{Rad}(\omega_i), \]
where $\mathcal{T}_{AE}(\omega)$ the tangent space of the AE for the region $\omega$, $\mathcal{X}$ denotes the data manifold, $d$ defines the 2-norm distance,and $\text{Rad}(\omega_i)$ the radius of the region $\omega_i$.
(Proof in Appendix~\ref{proof:tang}.)
\end{thm}

The previous statement shows that if the number of pieces of the piecewise affine map, which depends on the number of neurons in the DN architecture (see Fig.~\ref{fig:number_regions} and refer to \citet{montufar2014number} for more details) and the type of nonlinearity, goes to infinity, then the decoder would coincide with the data manifold. 
In a practical setting, it tells us that the higher the number of regions is, the higher is the degrees of freedom of the CPA, and that under this regularization, these degrees of freedom are controlled while not requiring more training points.

\section{Experimental Settings and Validations}
\label{sec:ex}
In this section, we discuss some practical aspects of the proposed regularizations as well as provide the experimental validations. In particular, how the parameters of the regularizations are learned (Sec.~\ref{sec:params}) as well as how the sampling required in both regularizations is performed along with their induced computational complexity (Sec.~\ref{sec:sampling}), and finally the experimental validations (Sec.~\ref{sec:exp}).

\subsection{Parameters Learning}
\label{sec:params}

\begin{table}[t]
\caption{ Comparison of the testing reconstruction errors ($\times 10^{-2} \pm \text{ std } \times 10^{-2})$ for each AE (columns) and dataset (rows). The methods denoted by $\textbf{Lie G. ($1^{th}$)}$ and $\textbf{Lie G. ($2^{nd}$)}$ correspond respectively to the first-order and second-order Lie group regularizations we developed. \textbf{H.O.C. AE} denotes the Higher-Order contractive AE, and \textbf{Den. AE} denoising AE.}
\setlength\tabcolsep{3.7pt}
\begin{tabular}{|l|c|c|c|c|c|}
\hline
\textit{Dataset \: \textbackslash \;  Model} & \textbf{AE}  & \textbf{Den. AE}  & \textbf{H.O.C. AE} & \textbf{Lie G. ($1^{th}$)}& \textbf{Lie G. ($2^{nd}$)} \\\hline \hline
\textbf{CIFAR10} & $5.6 \pm 0.05 $& $5.0 \pm 0.05$ & - & $\textbf{4.9} \pm 0.07$ & - \\ \hline 
\textbf{MNIST} &$12.01  \pm 0.003$ &   $12.01  \pm 0.004 $ & $12.01  \pm 0.004 $ & $\textbf{6.3} \pm 0.1$ &  $10.13  \pm 0.1  $\\ \hline
\textbf{CBF} & $62.38  \pm 0.74 $ & $52.66  \pm 0.76 $ & $ 51.09  \pm 0.54  $ & $\textbf{43.99} \pm 1.2$& $ 49.73  \pm 0.31   $\\  \hline
\textbf{Yoga} & $33.76  \pm 0.81 $ & $33.29  \pm 0.72$ & $ 32.08  \pm 0.42 $ & $\textbf{20.28} \pm 1.1$ & $30.78  \pm 1.2   $\\  \hline
\textbf{Trace} & $13.95  \pm 0.45 $ & $11.28  \pm 0.57$ & $ 12.57  \pm 0.21  $ & $13.23 \pm 0.4$ &$\textbf{10.91}  \pm 0.45 $\\  \hline
\textbf{Wine} & $63.06  \pm 0.02 $ & $59.34  \pm 0.02 $ & $ 49.94  \pm 0.02 $ & $\textbf{19.01} \pm 0.02$ & $49.94  \pm 0.01   $\\  \hline
\textbf{ShapesAll} & $67.98  \pm 3.0 $ & $58.67  \pm 1.4 $ & $ 61.42  \pm 5.5 $ & $\textbf{52.97} \pm 1.9$ & $57.80  \pm 1.2 $\\  \hline
\textbf{FiftyWords} & $64.91  \pm 1.7$ & $60.91  \pm 1.0 $ & $ 60.92  \pm 0.7 $ & $71.84 \pm 3.4$ &$\textbf{57.89}  \pm 1.0$\\  \hline
\textbf{WordSynonyms} & $70.95  \pm 1.5$ & $66.02  \pm 0.8 $ & $ 66.52  \pm 0.5$ & $68.21 \pm 2.7$ &$\textbf{62.22}  \pm 1.1$\\  \hline
\textbf{InsectSounds} & $51.86  \pm 0.6 $ & $40.24  \pm 0.8$ & $ 41.93  \pm 0.6 $ & $\textbf{38.11} \pm 0.9$ & $38.22  \pm 0.3 $\\  \hline
\textbf{ECG5000}  & $21.92  \pm 0.75 $  & $20.31 \pm 0.39$& $20.31  \pm 0.36$& $\textbf{18.06} \pm 0.9$ &$ 20.29  \pm 0.4 $\\ \hline
\textbf{Earthquakes}  & $56.23 \pm 4.1 $ &  $54.62  \pm 4.1 $ & $51.79 \pm 1.0 $ & $99.41 \pm 0.2$ &$\textbf{50.20} \pm 0.5$\\ \hline
\textbf{Haptics}  & $37.25 \pm 0.2 $ &  $36.02  \pm 1.8$ & $27.21 \pm 0.5 $ &$ \textbf{16.94} \pm 3.4$ & $26.06 \pm 0.9$\\ \hline
\textbf{FaceFour} & $49.82 \pm  1.0 $ &  $48.51  \pm 0.8 $ & $48.52 \pm 0.7 $ & $48.60 \pm 1.9 $ &$\textbf{46.00} \pm 0.6$\\ \hline   
\textbf{Synthetic} & $95.61 \pm  1.3 $ &  $89.37  \pm 1.0 $ & $88.47 \pm 0.9 $ & $\textbf{55.87} \pm 0.8 $ & $86.83 \pm 0.6 $\\ \hline  
\end{tabular}
\label{table:accu}
\end{table}

The degrees of freedom of our regularized AE comprise the usual AE parameters (per layer affine transformations) and the parameters of each regularization. The $\epsilon$ values are found from the analytical form given by Propositions~\ref{prop:theta_1} and \ref{prop:theta_2} in Appendix~\ref{app:params}. We learn the matrices $G_{k}, \forall k \in \left \{1,\dots,h \right \}$ with gradient descent based optimizer \citep{kingma2014adam} and thus our method introduces $h d^2$ additional parameters, where $d$ is the dimension of the input data. Note that a priori knowledge on the structure of the $G_k$ such as low-rank or skew-symmetric, i.e., Lie algebra of the special orthogonal group, can be imposed to reduce the number of parameters; we do not explore this in our study while it could be considered to speed up the computations and improve the regularization tractability. The regularizations themselves depend on the AE to find the optimal $\epsilon$ and adapt the matrices $G_k$. 
The dimension of each $G_k$ is quadratic in the dimension of the data. As such, for a high-dimensional datasets, the number of learnable parameters is large. Hence the optimization of the $G_k$ matrices remains the current bottleneck of the method.
We propose to apply the regularization term during training starting from the random initialization. More advanced strategies such as scheduled alternating minimization or employing a warm-up phase could be leveraged and result in further improvement in performance. 

We now describe how we approximate the regularization terms by removing the need to evaluate the entire input space partition of the decoder for the second-order term and removing the need to sample the entire domain for the first-order term.

\subsection{Region's Sampling and Computational Complexity}
\label{sec:sampling}
Recall that in the proposed second-order regularization, one should have the knowledge of the decoder latent space partition. In practice, and for large networks, the discovery of the partition would not be feasible. We thus propose to approximate the regularization by only sampling some of the regions and some of their respective neighbors.
This sampling is done by first randomly sampling some vectors in the AE latent space. As for each sample, the associated per region map is automatically formed during the foward pass of the decoder, the per region parameters can be obtained by computing the affine mapping induced by the samples. To compute the neighbors of those sample regions, we use a simple dichotomic search. That is, for each of the sampled regions, we sample another (nearby) vector and keep pushing this new sample toward the first sample until one obtains the closest sample that remains in a different region. With the above, one now has the knowledge of some regions and one neighboring region for each of those regions. We leverage this approach and perform the search of a single neighbor; for a better approximation of the regularization, one can repeat this sampling process and accumulate the obtained regions and neighbors. For the first-order term, we propose a similar approximation where we approximate the integral by sampling a latent space vector $\theta$ (at each mini-batch).

\begin{figure}[t]
\begin{center}
\hspace{1cm}
    \begin{minipage}{0.17\linewidth}
    \begin{center}
     AE
    \end{center}
    \end{minipage}
        \begin{minipage}{0.17\linewidth}
        \begin{center}
    Contractive 
    \\  AE
    \end{center}
    \end{minipage}
    \begin{minipage}{0.17\linewidth}
    \begin{center}
     Denoising \\
     AE
    \end{center}
    \end{minipage}
    \hspace{.05cm}
    \begin{minipage}{0.17\linewidth}
    \begin{center}
     Lie Group \\
    AE ($1^{th}$)
    \end{center}
    \end{minipage}
    \hspace{.05cm}
        \begin{minipage}{0.17\linewidth}
    \begin{center}
     Lie Group \\
    AE $(2^{nd})$
    \end{center}
    \end{minipage}
    \end{center}
    \begin{center}

    \centering
    \begin{minipage}{0.04\linewidth}
    \rotatebox{90}{Test Set Reconst. Error}
    \end{minipage}
    \hspace{-.5cm}
    \begin{minipage}{0.95\linewidth}
    \vspace{-.5cm}
    \includegraphics[width=1\linewidth]{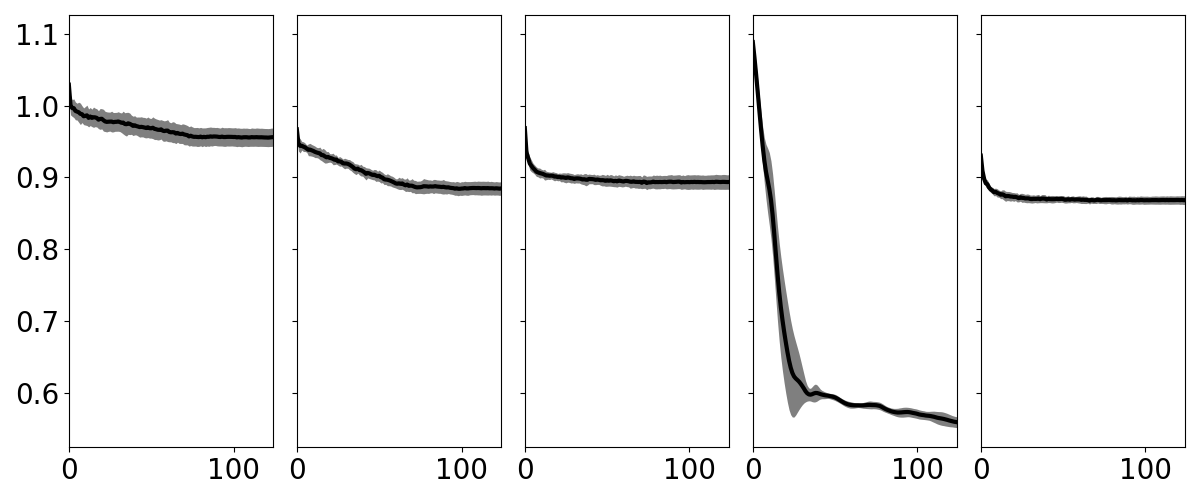}
        \vspace{-1.cm}
    \end{minipage}

\end{center}

\begin{center}
\begin{minipage}{.06\linewidth}
\begin{center}
    
\end{center}
\end{minipage}
    \begin{minipage}{0.16\linewidth}
    \begin{center}
     \; \;\;\; \;Epochs
    \end{center}
    \end{minipage}
        \hspace{.15cm}
    \begin{minipage}{0.16\linewidth}
    \begin{center}
    \;\;\;\;\;\; Epochs
    \end{center}
    \end{minipage}
            \hspace{.3cm}
        \begin{minipage}{0.16\linewidth}
        \begin{center}
    \; \; \;Epochs
    \end{center}
    \end{minipage}
                \hspace{.3cm}
    \begin{minipage}{0.16\linewidth}
    \begin{center}
    Epochs
    \end{center}
    \end{minipage}
                \hspace{.3cm}
    \begin{minipage}{0.16\linewidth}
    \begin{center}
    Epochs
    \end{center}
    \end{minipage}
    \end{center}
    \vspace{-.5cm}
       \caption{Test set reconstruction error on the SyntheticControl dataset evaluated on the best set of parameters for different AEs (from left to right): AE, Higher Order Contractive AE, Denoising AE,  Lie Group AE (first-order), and Lie Group AE (second-order). For each model, the mean over $10$ runs is reported in black, and the gray area corresponds to its standard deviation. We observe that the first order regularization performs much better than the second order one, which is close to the higher-order contactive AE error. In fact, this dataset contains six classes of time-series trends (upward, downward, normal,...), which can be easily related by a linear transformations, that is, the first order regularization can be easily optimized. The second order is harder to train and is more sensitive to the sampling of the regions, therefore is less reliable and harder to interpret. 
}
           \label{fig:explourd}
\end{figure}

\begin{figure}[t]
\begin{center}
    \begin{minipage}{.49\linewidth}
    \includegraphics[width=1\linewidth]{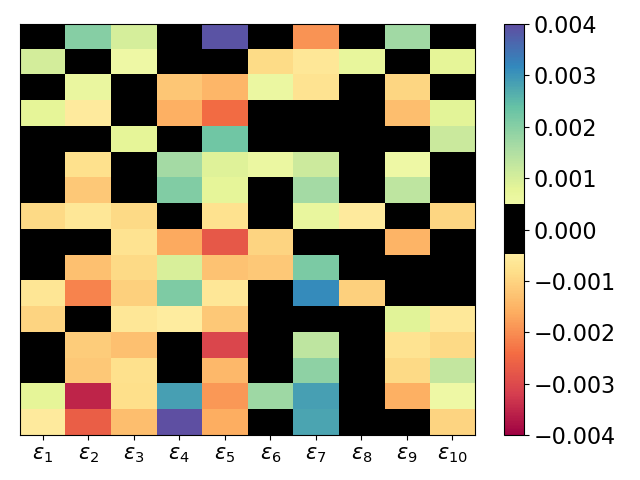}
    \end{minipage}
    \begin{minipage}{.49\linewidth}
    \includegraphics[width=1\linewidth]{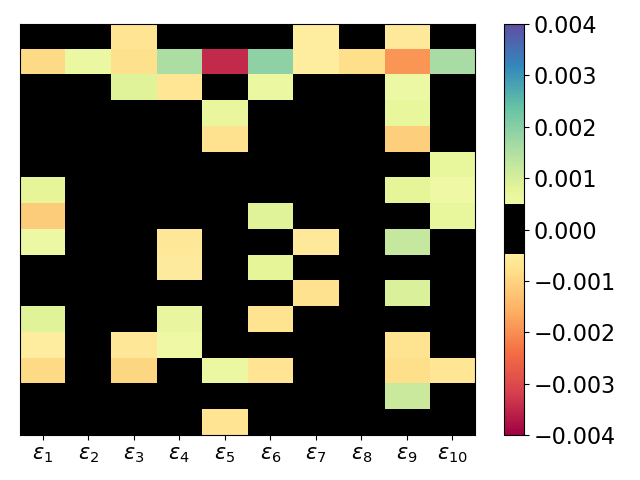}
    \end{minipage}
    \begin{minipage}{.49\linewidth}
    \includegraphics[width=1\linewidth]{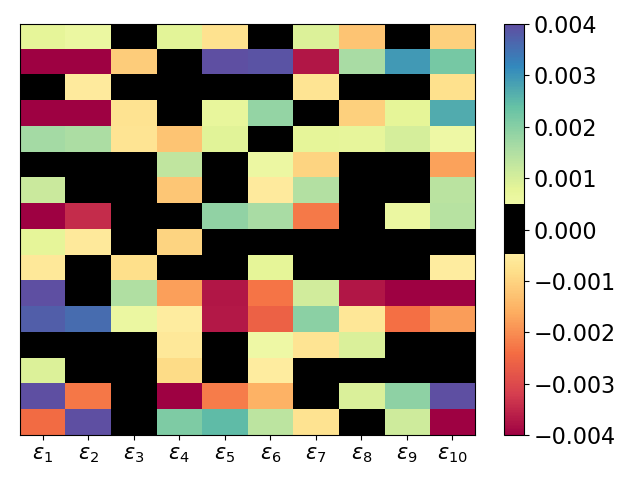}
    \end{minipage}
        \begin{minipage}{.49\linewidth}
        \caption{ Visualisation of the value (after learning) of the group strength parameter ,$\epsilon$, for the (\textit{Top left}) SyntheticControl, (\textit{Top right}) Haptics, and (\textit{Bottom left}) FaceFour datasets, in the case of the first order regularization. For each row, we observe the value of the the parameter $\epsilon_i$ for a given sample, $D(\theta)$ and its neigbhor  $D(\theta')$ as per Eq.~\ref{eq:regAE_first}, that is, what is the amplitude of the transformation associated to the generator $G_i$ needed to map one onto the other. Note that we display the epsilons for $16$ sampled neighboring pairs. } 
        \label{fig:epsilon}
        \end{minipage}
\end{center}
\end{figure}

Let us now consider the computational complexity induced by the regularizations omitting the computational cost of a pass through the AE as it is shared across all techniques. 
The optimal coordinates $\epsilon^*$ are obtained by solving a linear system of $h$ equations in both cases for each sampled datum (first-order) or each sampled region (second-order). This equation has to be solved for each sample region or latent space vector; we denote this by $N$ as in our case, we sample in each mini-batch as many vectors/regions as the size of the mini-batch. We obtain the time complexity $\mathcal{O}(h^2N+d^2hN)$ for the first-order, and $\mathcal{O}(h^2N+d^2h^2N)$ for the second-order, and a space complexity of $\mathcal{O}(d^2h)$ in both cases being driven by the need to retain the matrices $G_1,\dots,G_h$. The current bottleneck is the storage of those matrices, which limits the size of the AE bottleneck and output dimension. 

\subsection{Experiments}
\label{sec:exp}

We evaluate our framework on diverse datasets, including images and time-series data including speech, medical as well a seismic recordings, the description of the datasets is given in Appendix~\ref{ap:dataset}.
For each model and each hyperparameter, we perform $10$ runs for $125$ epochs with batch size $16$. The results are reported in Table~\ref{table:accu}. In this table, the statistics reported correspond to the average over the $10$ runs, each run using the test set performances based on the best validation set measure. Note that for CIFAR$10$, the computational burden of both the second-order Lie group regularization and the higher-order contractive one is too high. Thus only the AE, denoising AE and the first-order Lie group regularization are evaluated.

We propose, in particular, to visualize the test set reconstruction for the different AE models during training in Fig.~\ref{fig:explourd} (see Fig.~\ref{fig:error} in Appendix~\ref{app:Figures} for supplementary figures), where we can see that both Lie Group AEs are robust to the DN initialization and do not overfit. Besides, we can observe that while the first order on this dataset outperforms all the other regularizations, its variance at the begining of the learning phase is more volatile than other approaches. 

The hyperparameter responsible for the variance of the noise added to the data in the Denoising AE case also corresponding to the noise added to the data to sample Jacobian of nearby regions in Higher-Order Contractive AE parameter is evaluated for the values $\left \{0.001, 0.01,0.1,1 \right \}$. Another hyperparameter is the regularization trade-off parameter for both the Higher-Order Contractive AE and Lie Group AEs, the following values are tested for both models $\left \{0.001, 0.01, 0.1, 1, 10, 100,1000 \right \}$. All the models were trained using the same AE with $3$ fully connected encoder layers with ReLU with bottleneck dimension $h=10$, and $3$ fully connected decoder layer with ReLU and $1$ linear fully connected output layer.

We can observe in Table~\ref{table:accu} that the Lie group regularizations are usually outperforming the other methods the different datasets we evaluated. While the second-order regularization is more computationally demanding, it appears to be more stable and robust to the change of parameters. Besides, we can see that the first-order regularization might be more sensitive to how close to a Lie group the dataset under evaluation is. In fact, both the first and second-order depends on the learned infinitesimal operator, however, while the second-order use such matrix to constrain the overall curvature of the CPA the first-order use it to constrains its mapping.

In Fig.~\ref{fig:epsilon}, we show, for three datasets, the value of the parameters $\epsilon_i, \forall i \in \left \{1,\dots,h\right\}$ obtained by minimizing Eq.~\ref{eq:regAE_first}. This observation is important as the framework we propose assumes that the dimension of the orbit, hence the number of group transformations, is the same as the dimension of the bottleneck layer of the autoencoder. In particular, we propose to highlight the sparsity of the parameters, i.e., if for various sampled pairs, the value a particular strength parameter, $\epsilon_i$, is close to $0$. If it is, then the associated group of transformation is not being used to map any $\textbf{D}(\theta)$ to a close sample $\textbf{D}(\theta')$. Therefore, the number of group of transformations selected is too large, e.g., in the top right subplot, corresponding to the Haptics dataset, we observe that, the $10$ transformations are not required.

\section{Conclusion}

We analyzed AEs from a geometrical standpoint and provided insights into how AEs are approximating the data manifold. In particular, we provided analytical formulas of the per region map that AEs are performing using its continuous piecewise affine formulation. This approach's strength lies in its interpretability power, as for a given region in the input space, the DN mapping is a simple affine map. Leveraging these key features, we proposed to enhance and guarantee the generalization capability of AEs by proposing two regularizations that capture the symmetry in the data. These regularizations constrain the piecewise continuous surface spanned by the decoder to approximate the orbit of a Lie group. Besides, inspired by the theory of learning Lie group transformations, we alleviated the need to explicitly define a group of symmetry underlying the data and propose to learn the group's generator. In fact, the generator of a Lie group lives in a vector space, thus enabling common matrix manipulations required to perform its update.

Finally, we determined a computationally efficient procedure to learn these regularizations by performing a stochastic sampling approach to determine the per region parameters. Note that, in practice, this sampling is the bottleneck of our method since the regions do not form a regular tiling of the space. We showed that even using such an approximation, the method outperforms other existing regularizations and compete with state-of-the-art results on various datasets.

We believe our methodology opens the door to novel ways for DN practitioners to regularize the DN's parameter space depending on data assumptions. In particular, this approach should be relevant in the case of the application of DNs in science, where the underlying physical phenomenons can be partially theoretically prescribed, and following our approach, problem-specific regularization can be derived to provide the AE with generalization guarantees and improved approximation capabilities. 

\textbf{Acknowledgments:}{\small
 A special thanks to Anirvan Segupta and Yanis Barhoun for their insights and discussions.
RC and BA are supported by NSF grant SCH-1838873 and NIH grant R01HL144683-CFDA. Both RB are supported by NSF grants CCF-1911094, IIS-1838177, and IIS-1730574; ONR grants N00014-18-12571, N00014-20-1-2787, and N00014-20-1-2534; AFOSR grant FA9550-18-1-0478; and a Vannevar Bush Faculty Fellowship, ONR grant N00014-18-1-2047.}

\small
\bibliography{biblio}

\appendix
\clearpage

\section{Proofs}

\subsection{Proof of Proposition~\ref{prop_2}}
\label{proof:prop2}
\begin{proof}
Perfect reconstruction $\Rightarrow$: $\forall \omega, \forall x \in \omega$, $x= \sum_{k=1}^{h} \left \langle x, a^{E}_k[\omega] \right \rangle a^{D}_k[\omega]$.
We have $\forall \omega, \forall x \in \omega$
\begin{align*}
\sum_k \left \langle x, a^{E}_k[\omega]  \right \rangle a^{D}_k[\omega] & = \sum _{k=1}^{h} \left \langle \sum_{k'=1}^{h} \left \langle x, a^{E}_{k'}[\omega] \right \rangle a^{D}_{k'}[\omega] , a^{E}_k[\omega] \right \rangle a^{D}_k[\omega]  \\
& = \sum_k \sum_{k'=1}^{h} \left \langle x, a^{E}_{k'}[\omega] \right \rangle \left \langle a^{D}_{k'}[\omega], a^{E}_k[\omega] \right \rangle a^{D}_k[\omega] 
\end{align*}
$\iff A^{D}_{\omega} A^{E}_{\omega} x = A^{D}_{\omega} A^{D^{T}}_{\omega} A^{E^{T}}_{\omega} A^{E}_{\omega} x$ since $A^{D}_{\omega}A^{E}_{\omega}$ is injective on the region (as per perfect reconstruction condition) it implies that   $ A^{D^{T}}_{\omega} A^{E^{T}}_{\omega} = I_h$, where $I_h$ is the identity matrix of dimension $h \times h$
\end{proof}

\subsection{Proof of Corrolary~\ref{cor:corrolary_reconst}}
\label{proof:corrolary1}
\begin{proof}
For a $2$-layers ReLU autoencoder network, we have the following affine spline parameters $\forall x \in \omega$:
\begin{align*}
    a_{k'}^{E}[\omega]  = 1_{\left \{ W^{1^T}_{k',.} x >0 \right \} } W^{1}_{k',.} \\
    a_{k}^{D}[\omega]  = \begin{pmatrix}
    1_{\left \{ W^{2^T}_{1,.} \boldsymbol{z} >0 \right \} } \\
    \vdots \\
    1_{\left \{ W^{2^T}_{d,.} \boldsymbol{z} >0 \right \} } \\
    \end{pmatrix} 
    \cdot W^2_{.,k}
\end{align*}
where $\cdot$ defines here the elementwise vector multiplication.
Now, 
\begin{align*}
\left \langle a_k^D[\omega] , a^{E}_{k'}[\omega] \right \rangle & = \left \langle Q^2_{\omega} W^2_{.,k}, 1_{\left \{ W^{1^T}_{k',.} x >0 \right \}} W^{1}_{k',.} \right \rangle \\
& = 1_{ \left \{ W^{1^T}_{k',.} x >0  \right \} } W^{2^T}_{.,k} Q^2_{\omega} W^{1}_{k',.}  \\
&= 1_{ \left \{ W^{1^T}_{k',.} x>0 \right \} } W^{2^T}_{.,k} \begin{pmatrix}
1_{ \left \{ W^{2^T}_{1,.} \boldsymbol{E}(x)>0  \right \} } W^{1}_{k',1} \\
\vdots \\
1_{ \left \{ W^{2^T}_{n,.} \boldsymbol{E}(x)>0  \right \} } W^{1}_{k',d} \\
\end{pmatrix}  \\
& = 1_{ \left \{ W^{1^T}_{k',.} x>0 \right \} } \left ( \sum_{i=1}^{d} W^{2}_{i,k} W^{1}_{k',i} 1_{ \left \{ W^{2^T}_{i,.} \boldsymbol{E}(x)>0 \right \} } \right )
\end{align*}
\end{proof}

\subsection{Proof of Theorem~\ref{prop:oneisenough}}
\label{proof:cor1}
\begin{proof}
For both cases, we recall that we assume that $\forall \theta, x(\theta) \neq 0 $. In fact, relaxing such assumption would lead to a degenerated case where the interpolant can be constant and equal to $0$. In practice this assumption is more than realistic as the '$0$-datum' is usually not part of any dataset. 
Let's first consider the case $k=1$.
\\
\\
We know that the solution of $\frac{d\boldsymbol{f}(\theta)}{d \theta} =G\boldsymbol{f}$ is $\boldsymbol{f}(\theta) =  \exp(\theta G) \boldsymbol{f}(0)$. Now it is clear that if $\exists \theta_1$ such that $\boldsymbol{f}(\theta_1) = \boldsymbol{x}(\theta_1)$, then $\boldsymbol{f}(0) = \boldsymbol{x}(0)$, and therefore, $\boldsymbol{f}(\theta) = \exp(\theta G ) \boldsymbol{x}(0) = \boldsymbol{x}(\theta)$, $\forall \theta$.
\\
\\
Now for the case $k=2$, 
\\
\\
Let $y(\theta) = \frac{d f(\theta)}{d \theta}$, then we have
\[  \frac{ d \boldsymbol{y}(\theta)}{d \theta} = G \boldsymbol{y}(\theta), \]
which solution is
\[ \boldsymbol{y}(\theta) = \exp(\theta G) \boldsymbol{y}(0). \]
Thus, $\frac{d \boldsymbol{f}}{d \theta} = \exp(\theta G) \frac{d \boldsymbol{f}(\theta)}{d \theta }|_{\theta=0}$.
Now since 
\[\exp(\theta G) G\frac{d \boldsymbol{f}(\theta)}{d \theta }|_{\theta=0} =  \sum_{n \geq 0}  \frac{G^{n}}{n!} G \frac{d \boldsymbol{f}(\theta)}{d \theta }|_{\theta=0} = \sum_{n \geq 0} G \frac{G^{n}}{n!}  \frac{d \boldsymbol{f}(\theta)}{d \theta }|_{\theta=0} = G \exp(\theta G)  \frac{d \boldsymbol{f}(\theta)}{d \theta }|_{\theta=0}  \] 
we have that,
\[ \boldsymbol{f}(\theta) = \exp(\theta G) G^{-1} \frac{d \boldsymbol{f}(\theta)}{d \theta }|_{\theta=0} +c \boldsymbol{1}, \]
where $c \in \mathbb{R}$ and $\boldsymbol{1}$ denotes the $d$-dimensional vector of $1$.
Let's now add the interpolation condition, that is
\[ \exists \theta_1, \theta_2, \text{ s.t. } \boldsymbol{f}(\theta_1) = \boldsymbol{x}(\theta_1), \; \; \boldsymbol{f}(\theta_2) = \boldsymbol{x}(\theta_2) \]
Which is equivalent to
\[ \left\{\begin{matrix}
\exp(\theta_1 G)  G^{-1}\frac{d \boldsymbol{f}(\theta)}{d \theta }|_{\theta=0} + c \boldsymbol{1}= \exp(\theta_1G) x(0)   \\ 
\exp(\theta_2 G)  G^{-1}\frac{d \boldsymbol{f}(\theta)}{d \theta }|_{\theta=0} + c \boldsymbol{1}= \exp(\theta_2G) x(0)   \\ 
\end{matrix}\right. \]
Which implies that, $\frac{d \boldsymbol{f}(\theta)}{d \theta }|_{\theta=0}= Gx(0)$ and that $c=0$.

Therefore, \[ \boldsymbol{f}(\theta) = \exp(\theta G) x(0) = x(\theta), \forall \theta \]

\end{proof}

\subsection{Proof of Proposition~\ref{prop:theta_1}}
\label{proof:cortheta_1}
\begin{proof}
\begin{align*}
    \left \| \boldsymbol{D}(\theta) - (I+\sum_{k=1}^h \epsilon_k G_k) \boldsymbol{D}(\theta') \right \|_2^2 &= \left \| \boldsymbol{D}(\theta) - \boldsymbol{D}(\theta') - \sum_{k=1}^h \epsilon_k G_k \boldsymbol{D}(\theta') \right \|_2^2 \\
    & = \left \langle  \boldsymbol{D}(\theta) - \boldsymbol{D}(\theta')     ,      \boldsymbol{D}(\theta) - \boldsymbol{D}(\theta')    \right \rangle \\
    &  \; \;\;\;\;\; \;\;\;\;\; -2 \left \langle \boldsymbol{D}(\theta) - \boldsymbol{D}(\theta'), \sum_{k=1}^h \epsilon_k G_k \boldsymbol{D}(\theta') \right \rangle  \\
    & \; \; \; \;\;\;\;\;\;\;\;\;\; + \left \langle \sum_{k=1}^h \epsilon_k G_k \boldsymbol{D}(\theta'), \sum_{k=1}^h \epsilon_k G_k \boldsymbol{D}(\theta') \right \rangle,
\end{align*}
Now, $\forall j \in \left \{1,\dots, h \right \}$
\begin{align*}
\frac{\delta \left \|\boldsymbol{D}(\theta) - \boldsymbol{D}(\theta') - \sum_{k=1}^{h} \epsilon_k G_k \boldsymbol{D}(\theta') \right \|^2_{2}}{\delta \epsilon_j}  & = -2 (\boldsymbol{D}(\theta)-\boldsymbol{D}(\theta'))^T G_j \boldsymbol{D}(\theta') + 2 \sum_{k=1}^h \epsilon_k \boldsymbol{D}(\theta')^T G_k^T G_j \boldsymbol{D}(\theta'),
\end{align*}
setting $\frac{\delta \left \|\boldsymbol{D}(\theta) - \boldsymbol{D}(\theta') - \sum_{k=1}^{h} \epsilon_k G_k \boldsymbol{D}(\theta') \right \|^2_{2}}{\delta \epsilon_j} =0$, for all $j$ we obtain
\[ \epsilon^* = \begin{pmatrix}
\left \| G_1 \boldsymbol{D}(\theta') \right \|_2^2 & \dots & \boldsymbol{D}(\theta')^T G_h^T G_1 \boldsymbol{D}(\theta') \\
\vdots & \ddots & \vdots \\
\boldsymbol{D}(\theta')^T G_1^T G_h \boldsymbol{D}(\theta') & \dots & \left \| G_h \boldsymbol{D}(\theta')  \right \|_2^2 \end{pmatrix}^{-1}
\begin{pmatrix}
(\boldsymbol{D}(\theta) - \boldsymbol{D}(\theta'))^T G_1 \boldsymbol{D}(\theta') \\
\vdots \\
(\boldsymbol{D}(\theta) - \boldsymbol{D}(\theta'))^T G_h \boldsymbol{D}(\theta')
\end{pmatrix} \],
and we have that
\[ \begin{pmatrix}
\left \| G_1 \boldsymbol{D}(\theta') \right \|_2^2 & \dots & \boldsymbol{D}(\theta')^T G_h^T G_1 \boldsymbol{D}(\theta') \\
\vdots & \ddots & \vdots \\
\boldsymbol{D}(\theta')^T G_1^T G_h \boldsymbol{D}(\theta') & \dots & \left \| G_h \boldsymbol{D}(\theta')  \right \|_2^2 \end{pmatrix} = \begin{pmatrix}
G_1 \boldsymbol{D}(\theta') \\
\vdots \\
G_h \boldsymbol{D}(\theta') 
\end{pmatrix}^T \begin{pmatrix}
G_1 \boldsymbol{D}(\theta') \\
\vdots \\
G_h \boldsymbol{D}(\theta') 
\end{pmatrix} \]
which is thus a positive definite matrix.
\end{proof}

\subsection{Proof of Proposition~\ref{prop:theta_2}}
\label{proof:cortheta_2}
\begin{proof}
Given $J_{\omega'}[\boldsymbol{D}]= A^{D}_{\omega'} $ and $ J_{\omega}[\boldsymbol{D}] =  A^{D}_{\omega}$, we have
\begin{align*}
     \left \| A_{\omega'}^{D} - A_{\omega}^{D} - \sum_{k=1}^{h} \epsilon_k G_k A_{\omega}^{D} \right \|^2_{F} & = \Tr  ( (A_{\omega'}^{D} - A_{\omega}^D -\sum_{k=1}^{h} \epsilon_k G_k A^{D}_{\omega}) \odot (A_{\omega'}^{D} - A_{\omega}^D -\sum_{k=1}^{h} \epsilon_k G_k A^{D}_{\omega}) 1 1^T ) \\
     & = \text{Tr}  ( A^{D}_{\omega'} \odot A^{D}_{\omega'} - A^{D}_{\omega'} \odot A^{D}_{\omega} - A^{D}_{\omega'} \odot (\sum_{h=1}^{k} \epsilon_k G_k A^{D}_{\omega}) + A^{D}_{\omega} \odot A^{D}_{\omega} \\
     & \hspace{.5cm} \; \:  - A^{D}_{\omega} \odot A^{D}_{\omega'} + A^{D}_{\omega} \odot (\sum_{h=1}^{k} \epsilon_k G_k A^{D}_{\omega})- (\sum_{h=1}^{k} \epsilon_k G_k A^{D}_{\omega}) \odot A^{D}_{\omega'} \\
     & \hspace{.5cm} \; \: + (\sum_{h=1}^{k} \epsilon_k G_k A^{D}_{\omega}) \odot A^{D}_{\omega} + (\sum_{h=1}^{k} \epsilon_k G_k A^{D}_{\omega}) \odot (\sum_{h=1}^{k} \epsilon_k G_k A^{D}_{\omega}) ) 1 1^T ). 
\end{align*}
Now, $\forall j \in \left \{1,\dots,h \right \}$
\begin{align*}
     \frac{\delta \left \| A_{\omega'}^{D} - A_{\omega}^{D} - \sum_{k=1}^{h} \epsilon_k G_k A_{\omega}^{D} \right \|^2_{F}}{\delta \epsilon_j}  & = 2 \Tr ( (G_j A^{D}_{\omega}) \odot (A^{D}_{\omega} - A^{D}_{\omega'} + \sum_{k=1}^{h} G_k A^{D}_{\omega} 1 1^{T})) \\
     & = 2 \Tr (G_j A^{D}_{\omega} \odot (A^{D}_{\omega}- A^{D}_{\omega'}) 1 1^{T}) \\
     & \hspace{.5cm} + 2 \sum_{k=1}^{h} \epsilon_k \Tr( (G_j A^{D}_{\omega} \odot G_k A^{D}_{\omega}) 11^T), 
\end{align*}
setting $\frac{\delta \left \| A_{\omega'}^{D} - A_{\omega}^{D} - \sum_{k=1}^{h} \epsilon_k G_k A_{\omega}^{D} \right \|^2_{F}}{\delta \epsilon_j}=0$ for all $j$ and rearranging in matrix form gives

\[  \epsilon^* = \tiny{ \begin{pmatrix}
\sum_{i}\|G_1 [A_{\omega}^D]_{.,i}\|_2^2 & \dots & \sum_i \langle G_1 [A_{\omega}^D]_{.,i},G_h [A_{\omega}^D]_{.,i}\rangle \\ 
\vdots & \ddots & \vdots\\ 
\sum_i \langle G_h [A_{\omega}^D]_{.,i},G_1 [A_{\omega}^D]_{.,i}\rangle &  \dots &  \sum_{i}\|G_h [A_{\omega}^D]_{.,i}\|_2^2
\end{pmatrix}^{-1} \begin{pmatrix}
\sum_i \langle G_1 [A_{\omega}^D]_{.,i},[A_{\omega'}^D]_{.,i} - [A_{\omega}^D]_{.,i})  \\ 
\vdots \\ 
\sum_i \langle G_h [A_{\omega}^D]_{.,i},[A_{\omega'}^D]_{.,i} - [A_{\omega}^D]_{.,i})\end{pmatrix}} ,\]
and we have that 
\[  \tiny{ \begin{pmatrix}
\sum_{i}\|G_1 [A_{\omega}^D]_{.,i}\|_2^2 & \dots & \sum_i \langle G_1 [A_{\omega}^D]_{.,i},G_h [A_{\omega}^D]_{.,i}\rangle \\ 
\vdots & \ddots & \vdots\\ 
\sum_i \langle G_h [A_{\omega}^D]_{.,i},G_1 [A_{\omega}^D]_{.,i}\rangle &  \dots &  \sum_{i}\|G_h [A_{\omega}^D]_{.,i}\|_2^2
\end{pmatrix} = \sum_{i=1}^{h} \begin{bmatrix}
G_1 [A_{\omega}^D]_{.,i}\\ 
\vdots \\ 
G_h [A_{\omega}^D]_{.,i}
\end{bmatrix}^T \begin{bmatrix}
G_1 [A_{\omega}^D]_{.,i}\\ 
\vdots \\ 
G_h [A_{\omega}^D]_{.,i}
\end{bmatrix}}, \]

therefore it is the sum of positive definite matrices.

For the case h=1, we have that
\begin{align*}
    \left \| a_{\omega'}^{D} - a_{\omega}^{D} -  \epsilon G a_{\omega}^{D} \right \|^2 & = \left \langle a_{\omega'}^{D} , a_{\omega'}^{D}\right \rangle - 2 \left \langle a_{\omega'}^{D}, a_{\omega}^{D} \right \rangle + \left \langle a_{\omega}^{D}, a_{\omega}^{D} \right \rangle \\
    & \hspace{1cm} + 2 \left \langle \epsilon G a_{\omega}^{D}, a_{\omega}^{D} - a_{\omega'}^{D} \right \rangle + \left \langle \epsilon G a_{\omega}^{D}, \epsilon G a_{\omega}^{D} \right \rangle,
\end{align*}
thus,
\begin{align*}
    \frac{\delta \left \| a_{\omega'}^{D} - a_{\omega}^{D} -  \epsilon G a_{\omega}^{D} \right \|^2}{\delta \epsilon} = a_{\omega}^{D^T} G^T (a_{\omega}^{D}-a_{\omega'}^{D}) + \epsilon a_{\omega}^{D^T} G^T G a_{\omega}^{D}
\end{align*}
\end{proof}

For the following proofs, we will denote by $T: \mathbb{R}^{d} \times \mathbb{R}^{h} \rightarrow \mathbb{R}^d$, the transformation operator taking as input a datum and a group parameter, and giving as output the transformed datum. As we used a Lie group, we can define this operator analytically as $T(x,\theta)= \exp(\theta G) x$.

\subsection{Proof of Theorem~\ref{prop:tang}}
For this proof, we will use the notation $\mathcal{T}_{\mathcal{X}}(\omega)$ as the tangent space of the manifold described by the data $\mathcal{X}$ for the data in the region $\omega$, and by $\mathcal{T}_{AE}(\omega)$ the tangent space of the AE for the region $\omega$. We show that if these two tangent space coincides for a given region, i.e., if the tangent space of the AE coincides with the tangent space of the manifold for a specific position, then they coincide everywhere.
\label{proof:tang}
\begin{proof}
By assumption, we know that $\left \{ a^{D}_1(\omega'),\dots, a^{D}_h[\omega'] \right \}$ form a basis of $\mathcal{T}_{\mathcal{X}}[\omega']$. If the regularization is satisfied, we also know that the tangent induced by the AE at position $\omega$, denoted by $\mathcal{T}_{AE}(\omega)$, is equal to $T(\mathcal{T}_{\mathcal{X}}(\omega'),\theta)$. In fact, for the order $k=2$ the regularization imposes that the tangent (induced by the AE) of the different regions are transformed version of each other by the transformation operator $T$. Now for the order one, we know that if $\frac{d \boldsymbol{f}(\theta)}{d \theta} = G \boldsymbol{f}(\theta)$, then $\frac{d^2 \boldsymbol{f}(\theta)}{d \theta^2} = G \frac{d \boldsymbol{f}(\theta)}{d \theta}$. Which means that if the outputs of the interpolant $\boldsymbol{f}$ are connected by the transformation group $T$, then the tangents of such interpolant are also connected by the same group of transformation. 

Note that the operator $T$ forms a Lie group action operator, it is a diffeomorphism from the orbit of the group to the orbit of the group. Therefore, $\forall \omega$, it exists $\theta$ such that $T(\mathcal{T}_{\mathcal{X}}(\omega'),\theta)= \mathcal{T}_{\mathcal{X}}(\omega)$. Per assumption, the tangent of the region $\omega'$,i.e. $\mathcal{T}_{AE}(\omega')$ is actually tangent to $\mathcal{X}$ as its basis coincides with $\mathcal{T}_{\mathcal{X}}(\omega')$. Denote by $x \in \mathcal{X}$ the point at which $\mathcal{T}_{\mathcal{X}}(\omega')$ and $\mathcal{X}$ intersects. Let's first  first prove that for $\epsilon'= \argmax_{\epsilon} x + \epsilon h \in \omega$, where $h \in \mathcal{T}_{\mathcal{X}}(\omega')$, that is, $x+ \epsilon' h$ lies at the boundary of the region $\omega'$. We further assume that $\left \|h \right \|=1$ such that $\epsilon'= \text{Rad}(\omega')$. Let's define a smooth curve on the manifold $\gamma: \mathbb{R} \rightarrow \mathcal{X}$ such that $\gamma(0)=x$ and $\gamma'(0) = h$. Now,
\begin{align*}
    d(x+\epsilon' h,\mathcal{X}) & \leq d(x+\epsilon' h, \gamma(\epsilon')) \\
     & = \left \| \gamma(\epsilon') - \gamma(0) - \epsilon' \gamma'(0) \right \|.
\end{align*}
Since, $\lim_{\epsilon' \rightarrow 0} \frac{\gamma(\epsilon') - \gamma(0)}{\epsilon'} =\gamma'(0)$, we have that $\frac{d(x+\epsilon'h,\gamma(\epsilon'))}{\epsilon'} = o(\text{Rad}(\omega'))$. Then, since the $\omega_i \forall i \in \left \{1,\dots, |\Omega| \right \}$ form a partition of $\Omega$ and that by Proposition~\ref{prop:tang} we know that since one tangent of the AE coincides with the tangent of the manifold at the point $x$ then any tangent of the AE coincides with a tangent of the manifold. Thus, we have that $d \left ( \cup_{\omega \in \Omega} \mathcal{T}_{AE}(\omega), \mathcal{X} \right ) = \sum_{i=1}^{|\Omega|} d \left ( \mathcal{T}_{AE}(\omega_i), \mathcal{X} \right ) \leq \sum_{i=1}^{|\Omega|} \text{Rad}(\omega_i)$.

\end{proof}

\subsection{Per Region Tangent - Details}
\label{ap:jaco}
 Let $[\textbf{D} \circ \textbf{E}(.)]_i: \mathbb{R}^d \rightarrow \mathbb{R}$ be the $i^{th}$ coordinate output of the AE, defined as $[\textbf{D} \circ \textbf{E}(x)]_{i} = [A^D_{\omega}]_{i,.}A^{E}_{\omega}x + [A^{D}_{\omega}]_{i,.} B^{E}_{\omega} + [B^{D}_{\omega}]_{i}$.

\begin{align}
   \textit{d}[\textbf{D} \circ \textbf{E}(.)]_i &=[\textbf{D} \circ \textbf{E}(x+\epsilon)]_i-[(\textbf{D} \circ \textbf{E})(x)]_i 
    =\left \langle A_{\omega}^{{E}^T} [A^D_{\omega}]_{i,.}^T, \epsilon \right \rangle, \forall \epsilon \in \mathbb{R}^{d}.
\end{align}
As such, we directly obtain that  
\begin{align}
 \nabla_{x} [\textbf{D} \circ \textbf{E}(.)]_i = A_{\omega}^{{E}^T} [A^D_{\omega}]_{i,.}^T, 
\end{align}

which leads to the Jacobian of the AE as defined in Eq.~\ref{eq:jaco}.

\subsection{Higher-order Contractive Autoencoder: Details}
\label{sec:hoc}
Let consider the case of a 1 hidden-layer encoder, follows by any depth encoder. In the second order regularization, one penalizes $\left \|A^{E}_{\omega} - A^{E}_{\omega'} \right \|_{F}$, where $\omega$ and $\omega'$ are neighboring regions. We know that $A^{E}_{\omega}= Q^{1}_{\omega} W^1$, now let consider the case of a $3$ ReLU-units encoder, that is, $Q_{\omega}^{1}$ is a $3 \times 3$ diagonal matrix, and $W^1 \in \mathbb{R}^{3 \times n}$, where $n$ is the input space dimension. A particular case we consider for our analysis is, $Q^{1}_{\omega} = \text{Diag}(1,0,1)$, and $Q^{1}_{\omega'} = \text{Diag}(1,1,1)$, i.e., the first region $\omega$ is encoded by $2$ activated ReLUs and $\omega'$ by $3$. The associated HOC penalization is $\left \| \begin{pmatrix}
W^{1}_{1,:}\\ 
0\\ 
W^{1}_{3,:}
\end{pmatrix} - \begin{pmatrix}
W^{1}_{1,:}\\ 
W^{1}_{2,:}\\ 
W^{1}_{3,:}
\end{pmatrix} \right \|_{F} = \left \| W^{1}_{2,:} \right \|_2$, where $W^{1} = \begin{pmatrix}
W^{1}_{1,:}\\ 
W^{1}_{2,:}\\ 
W^{1}_{3,:}
\end{pmatrix}$.

Therefore we see even if $W^{1}_{1,:}$ or $W^{1}_{3,:}$ are large, they will not induce a penalization of the curvature between the region $\omega$ and $\omega'$. Besides, if $W^{1}_{2,:}$ is small, even thought it is associated with the changing unit between the two regions, the curvature will not be penalized either.

\section{Optimal Parameters}
\label{app:params}

The following proposition provides the optimal transformation parameters $\epsilon^{\star} = \left [\epsilon_1,\dots,\epsilon_h \right ]^T$ that are used during the learning of the regularized AE.
\begin{prop}
\label{prop:theta_1}
The $\epsilon$ of the first-order regularization defined in Eq.~\ref{eq:regAE_first} is obtained as 
\setlength\itemsep{.004em}
\[ \epsilon^* = \begin{pmatrix}
\left \| G_1 \boldsymbol{D}(\theta') \right \|_2^2 & \dots & \left \langle G_h \boldsymbol{D}(\theta') , G_1 \boldsymbol{D}(\theta') \right \rangle \\
\vdots & \ddots & \vdots \\
\left \langle G_1 \boldsymbol{D}(\theta'), G_h \boldsymbol{D}(\theta') \right \rangle & \dots & \left \| G_h \boldsymbol{D}(\theta')  \right \|_2^2 \end{pmatrix}^{-1}
\begin{pmatrix}
\left \langle \boldsymbol{D}(\theta) - \boldsymbol{D}(\theta'), G_1 \boldsymbol{D}(\theta') \right \rangle \\
\vdots \\
\left \langle \boldsymbol{D}(\theta) - \boldsymbol{D}(\theta'), G_h \boldsymbol{D}(\theta') \right \rangle 
\end{pmatrix}\]
where the matrix is always invertible ($\textbf{D}(\theta') \neq 0$).
(Proof in Appendix~\ref{proof:cortheta_1}.)
\end{prop}

\begin{prop}
\label{prop:theta_2}
The $\epsilon$ of the second-order regularization defined in Eq.~\ref{eq:regAE_second} is obtained by 
\[  \epsilon^* = \tiny{\begin{pmatrix}
\sum_{i}\|G_1 [A_{\omega}^D]_{.,i}\|_2^2 & \dots & \sum_i \langle G_1 [A_{\omega}^D]_{.,i},G_h [A_{\omega}^D]_{.,i}\rangle \\ 
\vdots & \ddots & \vdots\\ 
\sum_i \langle G_h [A_{\omega}^D]_{.,i},G_1 [A_{\omega}^D]_{.,i}\rangle &  \dots &  \sum_{i}\|G_h [A_{\omega}^D]_{.,i}\|_2^2
\end{pmatrix}^{-1} \begin{pmatrix}
\sum_i \langle G_1 [A_{\omega}^D]_{.,i},[A_{\omega'}^D]_{.,i} - [A_{\omega}^D]_{.,i})  \\ 
\vdots \\ 
\sum_i \langle G_h [A_{\omega}^D]_{.,i},[A_{\omega'}^D]_{.,i} - [A_{\omega}^D]_{.,i})\end{pmatrix}},\]
where the matrix is invertible ($A_{\omega}^{D} \neq 0$).
(Proof in Appendix~\ref{proof:cortheta_2}.)
\end{prop}

\section{Orbit of a Lie Group}
\label{Lieexample}
One example of the orbit of a data with respect to a Lie group is the result of the rotation on an initial point $x(0) \in \mathbb{R}^2$, we have $x(\theta) = \exp(\theta G)x(0), \theta \in \mathbb{R}, G =  \begin{pmatrix}
0 & -1 \\ 
1 & 0
\end{pmatrix}$. In fact, where we recall that $\exp(\theta \begin{pmatrix} 0 & -1 \\ 1 & 0 \end{pmatrix}) =  \begin{pmatrix}  \cos(\theta) & -\sin(\theta)\\  \sin(\theta) & \cos(\theta)  \end{pmatrix}.$ The infinitesimal operator $G$ is thus encapsulating the group information. For more details regarding Lie group and the exponential map refer to \cite{hall2015lie}.

\section{Datasets}
\label{ap:dataset}

Most of the datasets used for the experiments are extracted from the univariate time-series repository in \citet{bagnall16bakeoff}. Some of them are recording from sensors or simulated data. They range from motion time-series to biological one. The dimension of the data we used is between $200-1000$.

MNIST \citep{deng2012mnist}, is a handwritten digit dataset containing $60.000$ training and  $10.000$ test images of dimension $28 \times 28$ representing $10$ classes. This is an example of a dataset where the intra-class viariability is induced by group transformations such as rotation, translation, and small diffeomorphism

CIFAR$10$ \citep{krizhevsky2009learning}, consists of $50.000$ training colour images and $10.000$ test color images of dimension $32\times32$ representing $10$ classes. In this dataset, the transformation required to go from one data to another requires more than simple rigid transformations.

\clearpage
\section{Additional Experimental Figures}
\label{app:Figures}

\begin{figure}[th!]
\begin{center}
\hspace{1cm}
    \begin{minipage}{0.17\linewidth}
    \begin{center}
    AE
    \end{center}
    \end{minipage}
        \begin{minipage}{0.17\linewidth}
        \begin{center}
    Contractive 
    \\ AE
    \end{center}
    \end{minipage}
    \begin{minipage}{0.17\linewidth}
    \begin{center}
    Denoising \\
    AE
    \end{center}
    \end{minipage}
    \hspace{.05cm}
    \begin{minipage}{0.17\linewidth}
    \begin{center}
    Lie Group \\
    AE ($1^{th}$)
    \end{center}
    \end{minipage}
    \hspace{.05cm}
        \begin{minipage}{0.17\linewidth}
    \begin{center}
    Lie Group \\
    AE $(2^{nd})$
    \end{center}
    \end{minipage}
    \end{center}
    \begin{center}

    \centering
    \begin{minipage}{0.04\linewidth}
    \rotatebox{90}{Test Set Reconst. Error}
    \end{minipage}
    \begin{minipage}{0.95\linewidth}
    \includegraphics[width=1\linewidth]{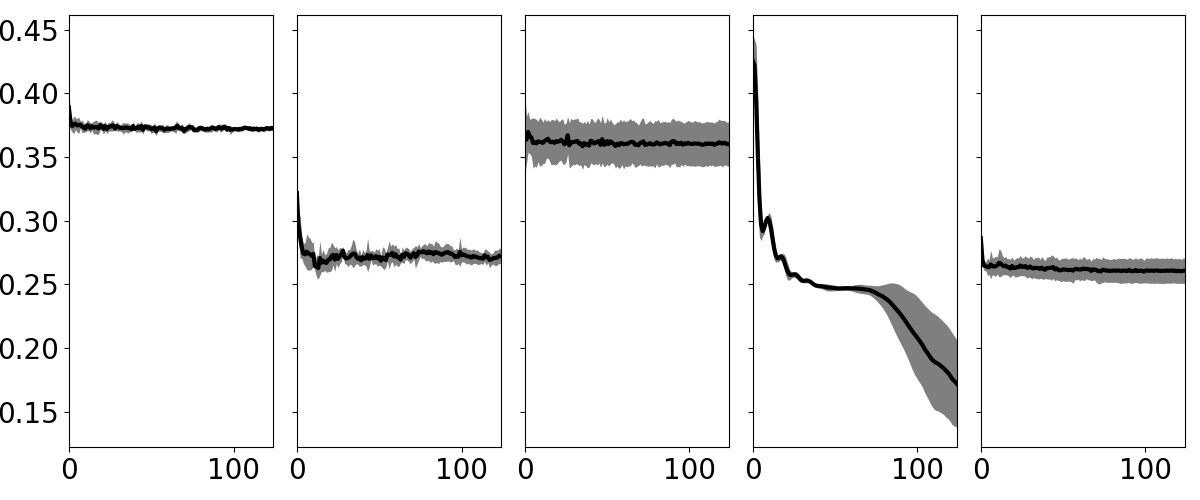}
    \end{minipage}

\end{center}
    \begin{center}

    \centering
    \begin{minipage}{0.04\linewidth}
    \rotatebox{90}{Test Set Reconst. Error}
    \end{minipage}
    \begin{minipage}{0.95\linewidth}
    \includegraphics[width=1\linewidth]{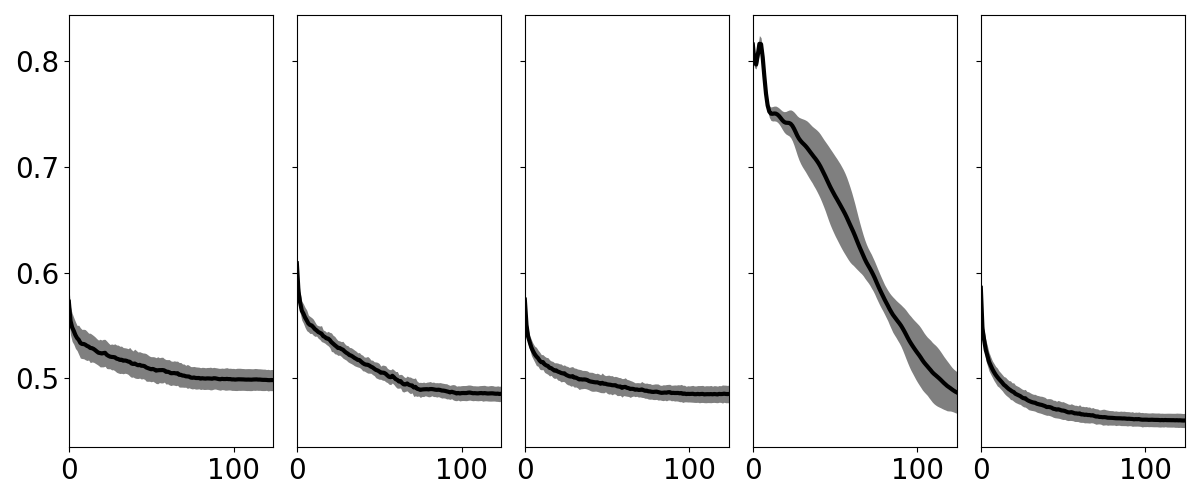}
    \end{minipage}

\end{center}

\begin{center}
\hspace{.5cm}
    \begin{minipage}{0.16\linewidth}
    \begin{center}
    Epochs
    \end{center}
    \end{minipage}
        \hspace{.05cm}
    \begin{minipage}{0.16\linewidth}
    \begin{center}
    Epochs
    \end{center}
    \end{minipage}
            \hspace{.1cm}
        \begin{minipage}{0.16\linewidth}
        \begin{center}
   Epochs
    \end{center}
    \end{minipage}
                \hspace{.2cm}
    \begin{minipage}{0.16\linewidth}
    \begin{center}
    Epochs
    \end{center}
    \end{minipage}
                \hspace{.2cm}
    \begin{minipage}{0.16\linewidth}
    \begin{center}
    Epochs
    \end{center}
    \end{minipage}
    \end{center}
    \caption{ 
    Test set reconstruction error on the (\textit{top}) Haptics and (\textit{bottom}) Facefour datasets evaluated on the best set of parameters for different AEs (from left to right): AE, Higher Order Contractive AE, Denoising AE, Lie Group AE (first-order), and Lie Group AE (second-order). For each model, the mean over $10$ runs is reported in black, and the gray area corresponds to its standard deviation.}
    \label{fig:error}
\end{figure}

\end{document}